\documentclass{article}

\usepackage{microtype}
\usepackage{graphicx}
\usepackage{subfigure}
\usepackage{booktabs} 
\usepackage{amssymb,amsmath,dsfont,multirow}
\usepackage[many]{tcolorbox}
\usepackage{amsthm}
\usepackage{thmtools}
\usepackage{tikz}
\usepackage{mathtools}

\usepackage{hyperref}

\usepackage[accepted]{icml2021}

\icmltitlerunning{Fast Margin Maximization via Dual Acceleration}

\numberwithin{equation}{section}
\declaretheorem[numberlike=equation]{theorem}
\declaretheorem[numberlike=theorem]{lemma}
\declaretheorem[numberlike=theorem]{proposition}

\declaretheoremstyle[%
qed={\ensuremath\Diamond}]{remstyle}

\usepackage[inline]{enumitem}
\usepackage{cleveref}
\usepackage{commath}
\usepackage{xspace}
\usepackage{nicefrac}
\def\R{\mathbb R}
\def\1{\mathds 1}

\newcommand{\ip}[2]{\left\langle #1, #2 \right \rangle}

\DeclareMathOperator*{\argmin}{arg\,min}
\DeclareMathOperator*{\argmax}{arg\,max}

\crefname{algocf}{alg.}{algs.}
\Crefname{algocf}{Algorithm}{Algorithms}
\newcommand\T{{\scriptscriptstyle\mathsf{T}}}
\def\margin#1{\gamma(#1)}
\def\nf{\nabla f}
\def\bargamma{\bar{\gamma}}
\def\mulgamma{\gamma_{\textup{m}}}
\def\bmgamma{{\bar\gamma}_{\textup{m}}}
\def\mulcR{\cR_{\textup{m}}}
\def\baru{\bar{u}}
\def\nR{\nabla\cR}

\def\barq{\bar{q}}
\def\nphi{\nabla\phi}
\def\tcO{\widetilde{\mathcal O}}
\def\tF{{\textrm{F}}}

\def\ddefloop#1{\ifx\ddefloop#1\else\ddef{#1}\expandafter\ddefloop\fi}
\def\ddef#1{\expandafter\def\csname bb#1\endcsname{\ensuremath{\mathbb{#1}}}}
\ddefloop ABCDEFGHIJKLMNOPQRSTUVWXYZ\ddefloop
\def\ddef#1{\expandafter\def\csname c#1\endcsname{\ensuremath{\mathcal{#1}}}}
\ddefloop ABCDEFGHIJKLMNOPQRSTUVWXYZ\ddefloop
\def\ddef#1{\expandafter\def\csname h#1\endcsname{\ensuremath{\hat{#1}}}}
\ddefloop ABCDEFGHIJKLMNOPQRSTUVWXYZ\ddefloop
\def\ddef#1{\expandafter\def\csname v#1\endcsname{\ensuremath{\boldsymbol{#1}}}}
\ddefloop ABCDEFGHIJKLMNOPQRSTUVWXYZabcdefghijklmnopqrstuvwxyz\ddefloop
\def\ddef#1{\expandafter\def\csname v#1\endcsname{\ensuremath{\boldsymbol{\csname #1\endcsname}}}}
\ddefloop {alpha}{beta}{gamma}{delta}{epsilon}{varepsilon}{zeta}{eta}{theta}{vartheta}{iota}{kappa}{lambda}{mu}{nu}{xi}{pi}{varpi}{rho}{varrho}{sigma}{varsigma}{tau}{upsilon}{phi}{varphi}{chi}{psi}{omega}{Gamma}{Delta}{Theta}{Lambda}{Xi}{Pi}{Sigma}{varSigma}{Upsilon}{Phi}{Psi}{Omega}{ell}\ddefloop

\begin{document}

\twocolumn[
\icmltitle{Fast Margin Maximization via Dual Acceleration}



\icmlsetsymbol{equal}{*}

\begin{icmlauthorlist}
\icmlauthor{Ziwei Ji}{uiuc}
\icmlauthor{Nathan Srebro}{ttic}
\icmlauthor{Matus Telgarsky}{uiuc}
\end{icmlauthorlist}

\icmlaffiliation{uiuc}{Department of Computer Science, University of Illinois at Urbana-Champaign, Urbana, Illinois, USA}
\icmlaffiliation{ttic}{Toyota Technical Institute of Chicago, Chicago, Illinois, USA}

\icmlcorrespondingauthor{Ziwei Ji}{ziweiji2@illinois.edu}

\icmlkeywords{Machine Learning, ICML}

\vskip 0.3in
]



\printAffiliationsAndNotice{}  

\begin{abstract}
We present and analyze a momentum-based gradient method for training linear
classifiers with an exponentially-tailed loss (e.g., the exponential or
logistic loss), which maximizes the classification margin on separable data
at a rate of $\widetilde{\mathcal{O}}(1/t^2)$.
This contrasts with a rate of $\mathcal{O}(1/\log(t))$ for standard
gradient descent, and $\mathcal{O}(1/t)$ for normalized gradient descent.  This
momentum-based method is derived via the convex dual of the maximum-margin
problem, and specifically by applying Nesterov acceleration to this dual,
which manages to result in a simple and intuitive method in the primal.
This dual view can also be used to
derive a stochastic variant, which performs adaptive non-uniform sampling via
the dual variables.
\end{abstract}

\section{Introduction}

First-order optimization methods,
such as stochastic gradient descent (SGD) and variants thereof,
form the optimization backbone of deep learning,
where they can find solutions with both low training error and low
test error \citep{nati_implicit_gen,zhang_gen}.
Motivated by this observation of low test error, there has been extensive work on the
\emph{implicit bias} of these methods: amongst those predictors with low training error,
which predictors do these methods implicitly prefer?

For linear classifiers and linearly separable data, \citet{nati_logistic} prove
that gradient descent can not only minimize the training error, but also
maximize the \emph{margin}.
This could help explain the good generalization of gradient descent, since a larger
margin could lead to better generalization \citep{spec}.
However, gradient descent can only maximize the margin at a slow $\cO(1/\log(t))$ rate.

\begin{figure}[t!]
  \centering
  \vspace{0.2em}
  \includegraphics[width=\linewidth]{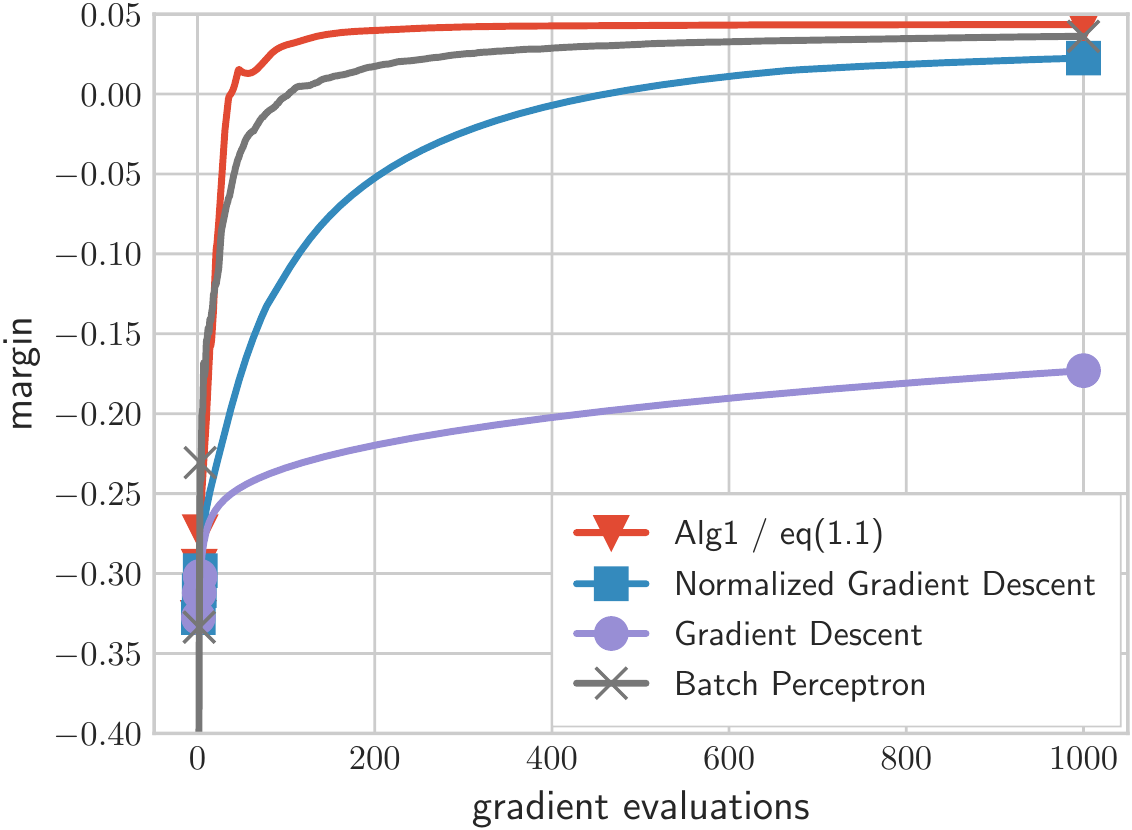}%
  \vspace{-1em}
  \caption{Margin-maximization performance of the new momentum-based method
    (cf.~\Cref{alg:momentron} and \cref{eq:nes}),
    which has a rate $\tcO(1/t^2)$, compared with prior work discussed below.
    All methods are first-order methods, and all but batch perceptron use an exponentially-tailed
    smooth loss, whereas batch perceptron applies gradient descent to the hard-margin problem
    directly.
    The data here is linearly separable, specifically \texttt{mnist} digits
  $0$ and $1$.}
  \label{fig:sep:batch}
\end{figure}

It turns out that the margin can be maximized much faster by simply normalizing
the gradient: letting $\theta_t$ denote the step size and $\cR$ the empirical risk
with the exponential loss, consider the normalized gradient step
\begin{align*}
  w_{t+1}:=w_t-\theta_t \frac{\nR(w_t)}{\cR(w_t)}.
\end{align*}
Using this normalized update,
margins are maximized at a $\widetilde{\cO}(1/\sqrt{t})$ rate with
$\theta_t=1/\sqrt{t}$ \citep{NLGSS18}, and at a $\cO(1/t)$ rate with $\theta_t=1$
\citep{refined_pd}.
A key observation in proving such rates is
that normalized gradient descent is equivalent to
an entropy-regularized mirror descent on a certain margin dual problem
(cf. \Cref{sec:dual_motiv}).

\paragraph{Contributions.}

In this work, we further exploit this duality relationship from prior work,
and design a momentum-based algorithm with iterates given by
  \begin{equation}
    \begin{aligned}
      g_t
      &:= \beta_t \del{ g_{t-1} + \frac{ \nabla \cR(w_t)}{\cR(w_t)}}, \\
      w_{t+1}
      &:= w_t - \theta_t \del{ g_t + \frac{ \nabla \cR(w_t)}{\cR(w_t)}}.
    \end{aligned}
    \label{eq:nes}
  \end{equation}
Our main result is that these iterates,
with a proper choice of $\theta_t$ and $\beta_t$, can maximize the margin at a rate of
$\widetilde{\cO}(1/t^2)$, whereas prior work had a rate of $\cO(1/t)$ at best.
The key idea is to reverse the primal-dual relationship mentioned above:
those works focus on primal normalized gradient descent, and show that it
is equivalent to dual mirror descent, but here we start from the dual,
and apply Nesterov acceleration to make dual optimization faster, and then
translate the dual iterates into the momentum form in \cref{eq:nes}.
Note that if our goal is just to accelerate dual optimization, then it is
natural to apply Nesterov's method; however, here our goal is to accelerate
(primal) margin maximization -- it was unclear whether the momentum method
changes the implicit bias, and our margin analysis is very different from the
standard analysis of Nesterov's method.
The connection between momentum in the primal and acceleration in the dual
also appears to be new, and we provide it as an auxiliary contribution.
We state the method in full in \Cref{alg:momentron}, and its analysis in \Cref{sec:momentron}.


Since our momentum-based iterates (cf. \cref{eq:nes}) are designed via a primal-dual framework,
they can be written purely with dual variables,
in which case they can be applied in a kernel setting.
However, calculating the full-batch gradient would require $n^2$ calls to the
kernel, where $n$ denotes the number of training examples.
To reduce this computational burden,
by further leveraging the dual perspective, we give an \emph{adaptive sampling}
procedure which avoids the earlier use of batch gradients and only needs $n$
kernel calls per iteration.
We prove a $\cO(1/\sqrt{t})$ margin rate for a momentum-free version of this
adaptive sampling method, but also provide empirical support for an aggressive
variant which uses our batch momentum formulation verbatim with these efficient stochastic updates.
These results are presented in \Cref{sec:adapt}.

For sake of presentation, the preceding analyses and algorithm definitions use the exponential
loss, however they can be extended to both binary and multiclass losses with
exponential tails.  The multiclass extension is in fact a straightforward reduction to
the binary case, and is used in most figures throughout this work.
We discuss these extensions in \Cref{sec:discussion}.

As an illustrative application of these fast margin maximization methods,
we use them to study the evolution of the kernel given by various stages of deep network
training.  The main point of interest is that while these kernels do seem to generally
improve during training (in terms of both margins and test errors), we provide an example where
simply changing the random seed switches between preferring the final kernel and the initial kernel.
These empirical results appear in \Cref{sec:emp}.

We conclude with open problems in \Cref{sec:open}.  Full proofs and further experimental
details are deferred to the appendices.

\subsection{Related Work}

This work is closely related to others on the implicit bias, most notably the original
analysis for gradient descent on linearly separable data \citep{nati_logistic}.
The idea of using normalized steps to achieve faster margin maximization rates
was first applied in the case of \emph{coordinate} descent \citep{mjt_margins},
where this normalization is closely associated with the usual step sizes in boosting
methods \citep{freund_schapire_adaboost}.
\citet{ramdas2016towards} studied a variant of the perceptron algorithm with
normalized steps, and showed it can always maximize the margin.
Many other works have used these normalized iterates, associated potential functions,
and duality concepts, both in the linear case
\citep{GLSS18,min_norm},
and in the nonlinear case
\citep{nati_lnn,kaifeng_jian_margin,chizat_bach_imp,dir_align}.

There appear to be few analyses of momentum methods; one example is the work of
\citet{heavy_ball_global}, which shows a $\cO(1/t)$ convergence rate for general
convex problems over bounded domains,
but can not be applied to the exponentially-tailed loss setting here
since the domain is unbounded and the solutions are off at infinity.
Connections between momentum in the primal and Nesterov acceleration in the dual
seem to not have been made before, and relatedly our use of momentum
coefficient $\beta_t = t/ (t+1)$ is non-standard.

Further on the topic of acceleration, \citet{tseng_agd} gave an application
to a smoothed version of the nonsmooth hard-margin objective, with a rate of $\cO(1/t)$ to
a fixed suboptimal margin.  This analysis requires accelerated methods for general geometries,
which were analyzed by \citet{tseng_agd} and \citet{allen_linear}.  The original accelerated
method for Euclidean geometry is due to \citet{nesterov_1983}.  A simultaneous analysis
of mirror descent and Nesterov acceleration is given here in \Cref{app_sec:unified}.

The methods here, specifically \Cref{fact:momentron_half},
can ensure a margin of $\bargamma/4$
in $4\sqrt{\ln(n)}/\bargamma$ steps, where $\bargamma$ denotes the optimal margin and will be
defined formally in \Cref{sec:notation}.
Another primal-dual method for fast \emph{linear feasibility} was given by
\citet{linear_optimism};
the method terminates in $\cO\del{\ln(n)/\bargamma}$ steps with a positive margin,
however the analysis does not reveal how large this margin is.

Various figures throughout this work include experiments with the \emph{batch perceptron},
which simply applies (super)gradient ascent to the explicit hard-margin maximization problem
\citep{batch_perceptron}.  Despite this simplicity, the method is hard to beat, and surpasses
prior implicit margin maximizers in experiments (cf. \Cref{fig:sep:batch}).
Interestingly, another standard method with strong guarantees
fared less well in experiments \citep{margin_lb},
and is thus omitted from the figures.



\section{Notation}
\label{sec:notation}

The dataset is denoted by $\{(x_i,y_i)\}_{i=1}^n$, where $x_i\in\R^d$ and
$y_i\in\{-1,+1\}$.
Without loss of generality, we assume $\|x_i\|_2\le1$.
Moreover, let $z_i:=-y_ix_i$, and collect these vectors into a matrix $Z\in\R^{n\times d}$,
whose $i$-th row is $z_i^\top$.

We consider linear classifiers.
The margin of a nonzero linear classifier $w\in\R^d$ is defined as
\begin{align*}
  \margin{w}:=\frac{\min_{1\le i\le n}y_i \langle w,x_i\rangle}{\|w\|_2}=\frac{-\max_{1\le i\le n}\langle w,z_i\rangle}{\|w\|_2},
\end{align*}
with $\margin{0}:=0$.
The \emph{maximum margin} is
\begin{align*}
  \bargamma:=\max_{\|w\|_2\le1}\margin{w}.
\end{align*}
If $\bargamma>0$, then the dataset is \emph{linearly separable}; in this case,
the \emph{maximum-margin classifier} is defined as
\begin{align*}
  \baru:=\argmax_{\|w\|_2\le1}\margin{w}=\argmax_{\|w\|_2=1}\margin{w}.
\end{align*}
If $\bargamma=0$, the dataset is linearly nonseparable.

Our algorithms are based on the empirical risk, defined as
\begin{align*}
  \cR(w):=\frac{1}{n}\sum_{i=1}^{n}\ell\del{\langle w,z_i\rangle}.
\end{align*}
For presentation, we mostly focus on the exponential loss $\ell(z):=e^z$, but our
analysis can be extended to other exponentially-tailed losses such as the
logistic loss $\ell(z):=\ln(1+e^z)$ and various multiclass losses; these extensions
are discussed in \Cref{sec:discussion}.

The following potential function $\psi:\R^n \to \R$ will be central to our analysis: given a strictly increasing
loss $\ell:\R\to\R$ with $\lim_{z\to-\infty}\ell(z)=0$ and
$\lim_{z\to\infty}\ell(z)=\infty$, for $\xi\in\R^n$, let
\begin{align}\label{eq:psi}
  \psi(\xi):=\ell^{-1}\del{\sum_{i=1}^{n}\ell(\xi_i)},
\end{align}
thus $\psi(Zw)=\ell^{-1}\del{n\cR(w)}$.
For the exponential loss,
$\psi$ is the ln-sum-exp function, meaning
$\psi(Zw)=\ln\del{\sum_{i=1}^{n}\exp(\langle w,z_i\rangle)}$.
This $\psi$ is crucial in our analysis since (i) it induces the dual
variable, which motivates our algorithms (cf. \Cref{sec:dual_motiv}); (ii) it
gives a smoothed approximation of margin, which helps in the margin analysis
(cf. \Cref{sec:momentron_margin}).
Here we note another useful property of $\psi$: the gradient of $\psi(Zw)$ with
respect to $w$ is $Z^\top\nabla\psi(Zw)$, which is a normalized version of
$\nR(w)$:
\begin{align}\label{eq:psi_grad}
  Z^\top\nabla\psi(Zw)=\frac{\sum_{i=1}^{n}\ell'\del{\langle w,z_i\rangle}z_i}{\ell'\del{\psi(Zw)}}=\frac{\nR(w)}{\ell'\del{\psi(Zw)}/n}.
\end{align}
For the exponential loss, $\nabla\psi(Zw)\in\Delta_n$ is just the softmax
mapping over $Zw$, where $\Delta_n$ denotes the probability simplex.
Moreover,
\begin{align}\label{eq:psi_exp}
  Z^\top\nabla\psi(Zw)=\frac{\nR(w)}{\cR(w)}.
\end{align}

\section{Analysis of \Cref{alg:momentron}}\label{sec:momentron}

\begin{figure}[t!]
  \centering
  \vspace{0.2em}
  \includegraphics[width=\linewidth]{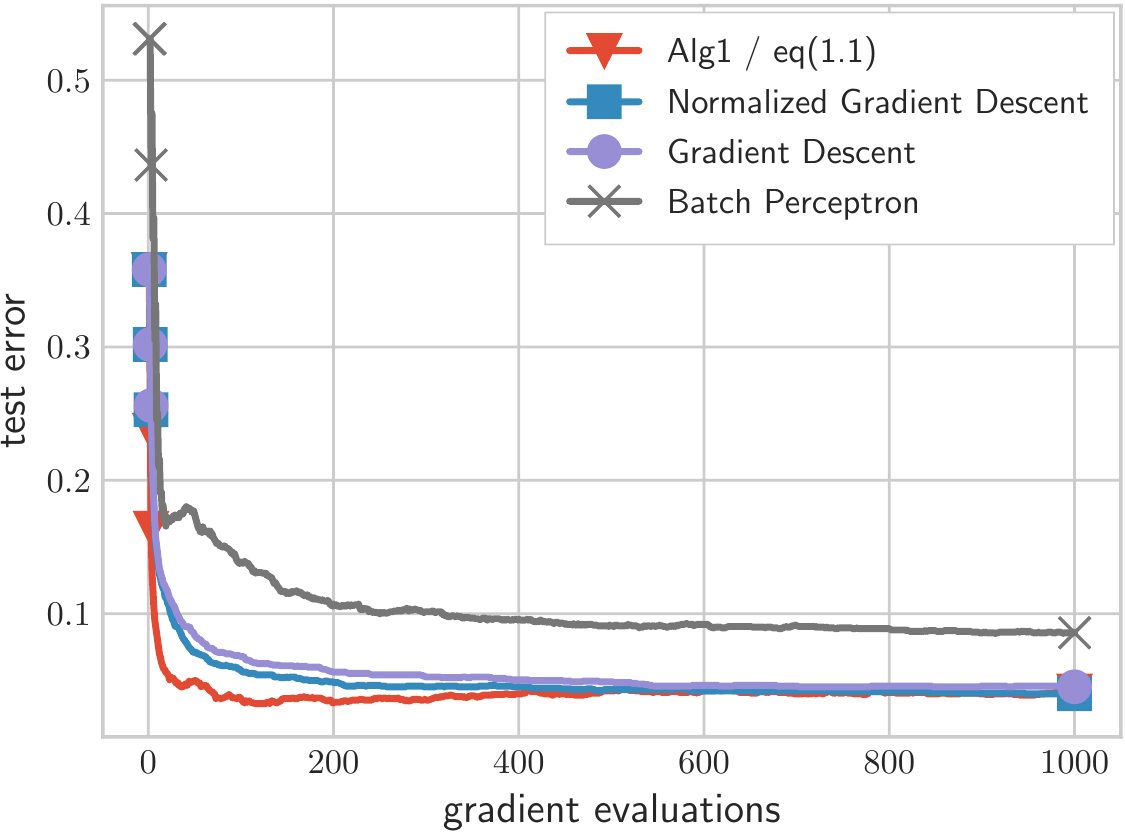}
  \vspace{-1.5em}
  \caption{Here the various margin-maximization methods from \Cref{fig:sep:batch} are run on
    \emph{non-separable} data, specifically \texttt{mnist} digits $3$ and $5$;
    as such, test error and not margin are reported.
    The methods based on exponential loss still perform well;
    by contrast, the batch perceptron suffers, and perhaps requires additional effort
  to tune a regularization parameter.}
  \label{fig:nonsep:batch}
\end{figure}

\begin{algorithm}[t!]
  \caption{}
  \label{alg:momentron}
  \begin{algorithmic}
    \STATE {\bfseries Input:} data matrix $Z\in\R^{n\times d}$, step size
    $(\theta_t)_{t=0}^\infty$, \\ momentum factor $(\beta_t)_{t=0}^\infty$.
    \STATE {\bfseries Initialize:}
    $w_0=g_{-1}=(0,\ldots,0)\in\R^d$, \\
    $q_0=(\frac{1}{n},\ldots,\frac{1}{n})\in\Delta_n$.
    \FOR{$t = 0,1,2,\ldots$}
      \STATE $g_t\gets\beta_t(g_{t-1}+Z^\top q_t)$.
      \STATE $w_{t+1}\gets w_t-\theta_t\del{g_t+Z^\top q_t}$.
      \STATE $q_{t+1}\propto\exp(Zw_{t+1})$, and $q_{t+1}\in\Delta_n$.
    \ENDFOR
  \end{algorithmic}
\end{algorithm}

A formal version of our batch momentum method is presented in
\Cref{alg:momentron}.
It uses the exponential loss, and is equivalent to \cref{eq:nes} since by
\cref{eq:psi_exp},
\begin{align*}
  Z^\top q_t=Z^\top\nabla\psi(Zw_t)=\frac{\nR(w_t)}{\cR(w_t)}.
\end{align*}
Here are our main convergence results.
\begin{theorem}\label{fact:momentron}
  Let $w_t$ and $g_t$ be given by \Cref{alg:momentron}
  with $\theta_t=1$ and $\beta_t=t/(t+1)$.
  \begin{enumerate}
    \item If the dataset is separable, then for all $t\ge1$,
    \begin{align*}
      \margin{w_t}\ge
      \bargamma-\frac{4\del{1+\ln(n)}\del{1+2\ln(t+1)}}{\bargamma(t+1)^2}.
    \end{align*}

    \item For any dataset, separable or nonseparable, it holds for all $t\ge1$
    that
    \begin{align*}
      \frac{4\|g_t\|_2^2}{t^2}-\frac{8\ln(n)}{(t+1)^2}\le\bargamma^2\le \frac{4\|g_t\|_2^2}{t^2}.
    \end{align*}
  \end{enumerate}
\end{theorem}

Our main result is in the separable case, where \Cref{alg:momentron} can
maximize the margin at a $\tcO(1/t^2)$ rate; by contrast, as mentioned in the
introduction, all prior methods have a $\cO(1/t)$ rate at best.
On the other hand, for any dataset, our algorithm can find an interval of
length $\cO(1/t^2)$ which includes $\bargamma^2$, in particular
certifying non-existence of predictors with margin larger than any value in this interval.
Moreover, as shown in \Cref{fig:nonsep:batch}, \Cref{alg:momentron} can also
achieve good test accuracy even in the nonseparable case; it is an interesting
open problem to build a theory for this phenomenon.

The rest of this section sketches the proof of \Cref{fact:momentron},
with full details deferred to the appendices.
In \Cref{sec:dual_motiv}, we first consider gradient descent without
momentum (i.e., $\beta_t=0$), which motivates consideration of a dual problem.
Then in \Cref{sec:pd_update}, we apply Nesterov acceleration
\citep{nesterov,tseng_agd,allen_linear} to this dual problem, and further derive
the corresponding primal method in \Cref{alg:momentron}, and also prove the second
part of \Cref{fact:momentron}.
Finally, we give a proof sketch of the margin rate in
\Cref{sec:momentron_margin}.

\subsection{Motivation from Gradient Descent}\label{sec:dual_motiv}

We start by giving an alternate presentation and discussion of certain observations
from the prior work of \citet{refined_pd}, which in turn motivates \Cref{alg:momentron}.

Consider gradient descent $w_{t+1}:=w_t-\eta_t\nR(w_t)$.
Define the dual variable by $q_t:=\nabla\psi(Zw_t)$; for the exponential loss,
it is given by $q_t\propto\exp(Zw_t)$, $q_t\in\Delta_n$.
Note that
\begin{align*}
  w_{t+1} & =w_t-\eta_t\nR(w_t) \\
   & =w_t-\eta_t\cR(w_t)\frac{\nR(w_t)}{\cR(w_t)} \\
   & =w_t-\theta_tZ^\top q_t,
\end{align*}
where we let $\theta_t=\eta_t\cR(w_t)$.
Moreover,
\begin{align*}
  q_{t+1}\propto\exp\del{Zw_{t+1}} & =\exp\del{Zw_t-\theta_tZZ^\top q_t} \\
   & \propto q_t\odot\exp\del{-\theta_tZZ^\top q_t} \\
   & =q_t\odot\exp\del{-\theta_t\nphi(q_t)},
\end{align*}
where $\phi(q):=\enVert{Z^\top q}_2^2/2$ and $\odot$ denotes coordinate-wise
product.
In other words, the update from $q_t$ to $q_{t+1}$ is a mirror descent / dual
averaging update with the entropy regularizer on the dual objective $\phi$.

This dual objective $\|Z^\T q\|_2^2/2$ is related to the usual hard-margin dual objective,
and is evocative of the SVM dual problem; this connection is made explicit in \Cref{sec:dual}.
Even without deriving this duality formally, it makes sense that $q_t$ tries to minimize $\phi$,
since $\phi$
encodes extensive structural information of the problem:
for instance, if the dataset is not separable, then $\min_{q\in\Delta_n}\phi(q)=0$ (cf.
\Cref{fact:margin_pd}).
With a proper step size, we can ensure
\begin{align*}
  \phi(q_t)=\frac{\enVert{\nR(w_t)}_2^2}{2\cR(w_t)^2}\to0,\quad\cR(w_t)\textup{ is
  nonincreasing},
\end{align*}
and it follows that $\enVert{\nR(w_t)}_2\to0$.
If the dataset is separable, then $\min_{q\in\Delta_n}\phi(q)=\bargamma^2/2$
(cf. \Cref{fact:margin_pd}), and
\begin{align*}
  Z^\top\barq=\bargamma\baru,\quad\textup{for }\barq\in\argmin_{q\in\Delta_n}\phi(q),
\end{align*}
where $\baru$ is the unique maximum-margin predictor, as defined in \Cref{sec:notation}.
As $q_t$ minimizes $\phi$, the vector $Z^\top q_t$ becomes biased towards $\baru$,
by which we can also show $w_t/\|w_t\|_2\to\baru$.
\citet{refined_pd} use this idea to show a $\cO(1/t)$ margin
maximization rate for primal gradient descent.

The idea in this work is to reverse the above process: we can start from the dual
and aim to minimize $\phi$ more efficiently,
and then take the dual iterates $(q_t)_{t=0}^\infty$
from this more efficient minimization and use them to
construct primal iterates $(w_t)_{t=0}^\infty$ satisfying
$\nabla\psi(Zw_t)=q_t$.
It is reasonable to expect such $w_t$ to maximize the margin faster, and indeed
we show this is true in the following, by applying Nesterov acceleration to the
dual, thanks to the $\ell_1$ smoothness of $\phi$ \citep[Lemma 2.5]{refined_pd}.

\subsection{Primal and Dual Updates}\label{sec:pd_update}

To optimize the dual objective $\phi$, we apply Nesterov's method with the
$\ell_1$ geometry \citep{tseng_agd,allen_linear}.
The following update uses the entropy regularizer; more general updates are
given in \Cref{app_sec:unified}.

Let $\mu_0=q_0:=(\frac{1}{n},\ldots,\frac{1}{n})$.
For $t\ge0$, let $\lambda_t,\theta_t\in(0,1]$, and
\begin{align*}
  \nu_t & :=(1-\lambda_t)\mu_t+\lambda_tq_t, \\
  q_{t+1} & \propto q_t\odot\exp\del{-\frac{\theta_t}{\lambda_t}ZZ^\top\nu_t},\quad q_{t+1}\in\Delta_n, \\
  \mu_{t+1} & :=(1-\lambda_t)\mu_t+\lambda_tq_{t+1}.
\end{align*}
If we just apply the usual mirror descent / dual averaging to $\phi$, then $\phi$
can be minimized at a $\cO(1/t)$ rate \citep[Theorem 2.2]{refined_pd}.
However, using the above accelerated process, we can minimize $\phi$ at a $\cO(1/t^2)$ rate.

\begin{lemma}\label{fact:dual_phi_simple}
  For all $t\ge0$, let $\theta_t=1$ and $\lambda_t=2/(t+2)$.
  Then for all $t\ge1$ and $\barq\in\argmin_{q\in\Delta_n}\phi(q)$,
  \begin{align*}
    \phi(\mu_t)-\phi(\barq)
    \le \frac{4\ln(n)}{(t+1)^2}.
  \end{align*}
\end{lemma}

Next we construct corresponding primal variables $(w_t)_{t=0}^\infty$ such that
$\nabla\psi(Zw_t)=q_t$.
(We do not try to make $\nabla\psi(Zw_t)=\nu_t$ or $\mu_t$, since only $q_t$ is
constructed using a mirror descent / dual averaging update.)
Let $w_0:=0$, and for $t\ge0$, let
\begin{align}\label{eq:primal_var}
  w_{t+1}:=w_t-\frac{\theta_t}{\lambda_t}Z^\top\nu_t.
\end{align}
We can verify that $q_t$ is indeed the dual variable to $w_t$, in the sense that
$\nabla\psi(Zw_t)=q_t$: this is true by definition at $t=0$, since
$\nabla\psi(Zw_0)=\nabla\psi(0)=q_0$.
For $t\ge0$, we have
\begin{align*}
  q_{t+1} & \propto q_t\odot\exp\del{-\frac{\theta_t}{\lambda_t}ZZ^\top\nu_t} \\
   & \propto\exp(Zw_t)\odot\exp\del{-\frac{\theta_t}{\lambda_t}ZZ^\top\nu_t} \\
   & =\exp\del{Z\del{w_t-\frac{\theta_t}{\lambda_t}Z^\top\nu_t}}=\exp(Zw_{t+1}).
\end{align*}
In addition, we have the following characterization of $w_t$ based on a momentum
term, giving rise to the earlier \cref{eq:nes}.
\begin{lemma}\label{fact:wt_characterization_simple}
  For all $\lambda_t,\theta_t\in(0,1]$, if $\lambda_0=1$, then for all $t\ge0$,
  \begin{align*}
    w_{t+1}=w_t-\theta_t\del{g_t+Z^\top q_t},
  \end{align*}
  where $g_0:=0$, and for $t\ge1$,
  \begin{align*}
    g_t:=\frac{\lambda_{t-1}(1-\lambda_t)}{\lambda_t}\del{g_{t-1}+Z^\top q_t}.
  \end{align*}
  Specifically, for $\lambda_t=2/(t+2)$, it holds that
  \begin{align*}
    \frac{\lambda_{t-1}(1-\lambda_t)}{\lambda_t}=\frac{t}{t+1},\quad\textup{and}\quad g_t=\sum_{j=1}^{t}\frac{j}{t+1}Z^\top q_j,
  \end{align*}
  and $Z^\top\mu_t=2g_t/t$.
\end{lemma}
Consequently, with $\lambda_t=2/(t+2)$, the primal iterate defined by
\cref{eq:primal_var} coincides with the iterate given by \Cref{alg:momentron}
with $\beta_t=t/(t+1)$.

Additionally, \Cref{fact:dual_phi_simple,fact:wt_characterization_simple}
already prove the second part of \Cref{fact:momentron}, since
$\phi(\mu_t)=4\|g_t\|_2^2/(2t^2)$ by \Cref{fact:wt_characterization_simple},
while $\phi(\barq)=\bargamma^2/2$ by \Cref{fact:margin_pd}.

\subsection{Margin Analysis}\label{sec:momentron_margin}

Now we consider the margin maximization result of \Cref{fact:momentron}.
The function $\psi$ will be important here, since it gives a smoothed
approximation of the margin: recall that $\psi(Zw)$ is defined as
\begin{align*}
  \psi(Zw)=\ell^{-1}\del{\sum_{i=1}^{n}\ell\del{\langle z_i,w\rangle}}.
\end{align*}
Since $\ell$ is increasing, we have
\begin{align*}
  -\psi(Zw)\le & -\ell^{-1}\del{\max_{1\le i\le n}\ell\del{\langle z_i,w\rangle}} \\
   & =-\ell^{-1}\del{\ell\del{\max_{1\le i\le n}\langle z_i,w\rangle}} \\
   & =-\max_{1\le i\le n}\langle z_i,w\rangle.
\end{align*}
As a result, to prove a lower bound on $\margin{w_t}$, we only need to prove a
lower bound on $-\psi(Zw_t)/\|w_t\|_2$, and it would be enough if we have a
lower bound on $-\psi(Zw_t)$ and an upper bound on
$\|w_t\|_2$.

Below is our lower bound on $-\psi$ for \Cref{alg:momentron}.
Its proof is based on a much finer analysis of dual Nesterov, and uses both
primal and dual smoothness.
\begin{lemma}\label{fact:-psi_lb_agd_simple}
  Let $\theta_t=1$ for all $t\ge0$, and $\lambda_0=1$, then for all $t\ge1$,
  \begin{align*}
    -\psi(Zw_t)\ge & \ -\psi(Zw_0)+\frac{1}{2\lambda_{t-1}^2}\enVert{Z^\top\mu_t}_2^2 \\
     & \ +\sum_{j=1}^{t-1}\frac{1}{2}\del{\frac{1}{\lambda_{j-1}^2}-\frac{1-\lambda_j}{\lambda_j^2}}\enVert{Z^\top\mu_j}_2^2 \\
     & \ +\sum_{j=0}^{t-1}\frac{1}{2\lambda_j}\enVert{Z^\top\nu_j}_2^2.
  \end{align*}
\end{lemma}

Additionally, here are our bounds on $\|w_t\|_2$.
\begin{lemma}\label{fact:wt_norm_simple}
  Let $\theta_t=1$ for all $t\ge0$, then
  \begin{align*}
    \sum_{j=0}^{t-1}\frac{\bargamma}{\lambda_j}\le\|w_t\|_2\le \sum_{j=0}^{t-1}\frac{1}{\lambda_j}\enVert{Z^\top\nu_j}_2.
  \end{align*}
\end{lemma}

With \Cref{fact:-psi_lb_agd_simple,fact:wt_norm_simple}, we can prove
\Cref{fact:momentron}.
Here we show a weaker result which gives $1/t^2$ convergence to $\bargamma/2$;
its proof is also part of the full proof of \Cref{fact:momentron}, but much
simpler.
The remaining proof of \Cref{fact:momentron} is deferred to
\Cref{app_sec:momentron}.

\begin{proposition}[weaker version of \Cref{fact:momentron}]\label{fact:momentron_half}
  With $\theta_t=1$ and $\lambda_t=2/(t+2)$, we have
  \begin{align*}
    \margin{w_t}\ge \frac{\bargamma}{2}-\frac{4\ln(n)}{\bargamma(t+1)^2}.
  \end{align*}
\end{proposition}
\begin{proof}
  With $\lambda_t=2/(t+2)$, it holds that
  \begin{align*}
    \frac{1}{\lambda_{j-1}^2}-\frac{1-\lambda_j}{\lambda_j^2}\ge0,
  \end{align*}
  therefore
  \begin{align}\label{eq:momentron_half_tmp1}
    -\psi(Zw_t)\ge-\psi(Zw_0)+\sum_{j=0}^{t-1}\frac{1}{2\lambda_j}\enVert{Z^\top\nu_j}_2^2.
  \end{align}
  Then \cref{eq:momentron_half_tmp1} and \Cref{fact:wt_norm_simple} imply
  \begin{align}\label{eq:momentron_half_tmp3}
    \frac{\psi(Zw_0)-\psi(Zw_t)}{\|w_t\|_2}\ge \frac{\sum_{j=0}^{t-1}\frac{1}{2\lambda_j}\enVert{Z^\top\nu_j}_2^2}{\sum_{j=0}^{t-1}\frac{1}{\lambda_j}\enVert{Z^\top\nu_j}_2}\ge \frac{\bargamma}{2},
  \end{align}
  since $\enVert{Z^\top\nu_j}_2\ge\bargamma$ (cf. \Cref{fact:margin_pd}).
  On the other hand, \Cref{fact:wt_norm_simple} and $\lambda_t=2/(t+2)$ imply
  \begin{align*}
    \|w_t\|_2\ge \sum_{j=0}^{t-1}\frac{\bargamma}{\lambda_j}\ge \frac{\bargamma(t+1)^2}{4},
  \end{align*}
  and thus
  \begin{align}\label{eq:momentron_half_tmp4}
    \frac{\psi(Zw_0)}{\|w_t\|_2}=\frac{\ln(n)}{\|w_t\|_2}\le \frac{4\ln(n)}{\bargamma(t+1)^2}.
  \end{align}
  It then follows from \cref{eq:momentron_half_tmp3,eq:momentron_half_tmp4} that
  \begin{align*}
    \margin{w_t}\ge\frac{-\psi(Zw_t)}{\|w_t\|_2}\ge \frac{\bargamma}{2}-\frac{4\ln(n)}{\bargamma(t+1)^2}.
  \end{align*}
\end{proof}

\section{Analysis of \Cref{alg:adaptron}}\label{sec:adapt}

Since \Cref{alg:momentron} is derived from dual Nesterov, it can also be run
completely in the dual, meaning primal iterates and in particular the primal dimensionality
never play a role.
However, this dual version would require calculating $ZZ^\top q_t$,
which in the kernel setting requires $n^2$ kernel calls.
In \Cref{alg:adaptron}, we replace $Z^\top q_t$ with a single column $z_{i_t}$ of $Z^\top$, where
$i_t$ is sampled from $q_t\in\Delta_n$.  This sampling allows us to make only
$n$ kernel calls per iteration, rather than $n^2$ as in \Cref{alg:momentron}.

\begin{algorithm}[t!]
  \caption{}
  \label{alg:adaptron}
  \begin{algorithmic}
    \STATE {\bfseries Input:} data matrix $Z\in\R^{n\times d}$, step size
    $(\theta_t)_{t=0}^\infty$, \\ momentum factor $(\beta_t)_{t=0}^\infty$.
    \STATE {\bfseries Initialize:}
    $w_0=g_{-1}=(0,\ldots,0)\in\R^d$, \\
    $q_0=(\frac{1}{n},\ldots,\frac{1}{n})\in\Delta_n$.
    \FOR{$t = 0,1,2,\ldots$}
      \STATE Sample $i_t\sim q_t$.
      \STATE $g_t\gets\beta_t\del{g_{t-1}+z_{i_t}}$.
      \STATE $w_{t+1}\gets w_t-\theta_t\del{g_t+z_{i_t}}$.
      \STATE $q_{t+1}\propto\exp(Zw_{t+1})$, and $q_{t+1}\in\Delta_n$.
    \ENDFOR
  \end{algorithmic}
\end{algorithm}

Unfortunately, we do not have a general theory for \Cref{alg:adaptron}.
Instead, as follows, we provide here an analysis with momentum disabled, meaning $\beta_t=0$,
and a small constant step size $\theta_t$.
\begin{theorem}\label{fact:primal_margin_sgd}
  Given $\epsilon>0$ and $\delta\in(0,1)$, let
  \begin{align*}
    t=\max\del{\left\lceil \frac{32\ln(n)+64\ln(2/\delta)}{\bargamma^2\epsilon^2}\right\rceil,\left\lceil\frac{32}{\delta\epsilon^2}\right\rceil},
  \end{align*}
  and $\theta_j=\sqrt{\ln(n)/t}$ for $0\le j<t$, then with probability
  $1-\delta$,
  \begin{align*}
    \margin{w_t}\ge\bargamma-\epsilon.
  \end{align*}
\end{theorem}

The proof of \Cref{fact:primal_margin_sgd} is similar to the proof of \Cref{fact:momentron},
but must additionally produce high-probability bounds on
$-\psi(Zw_t)$ and $\|w_t\|_2$; details are deferred to \Cref{app_sec:adapt}.

\begin{figure}[t!]
  \centering
  \subfigure[Margins.]{%
    \includegraphics[width=\linewidth]{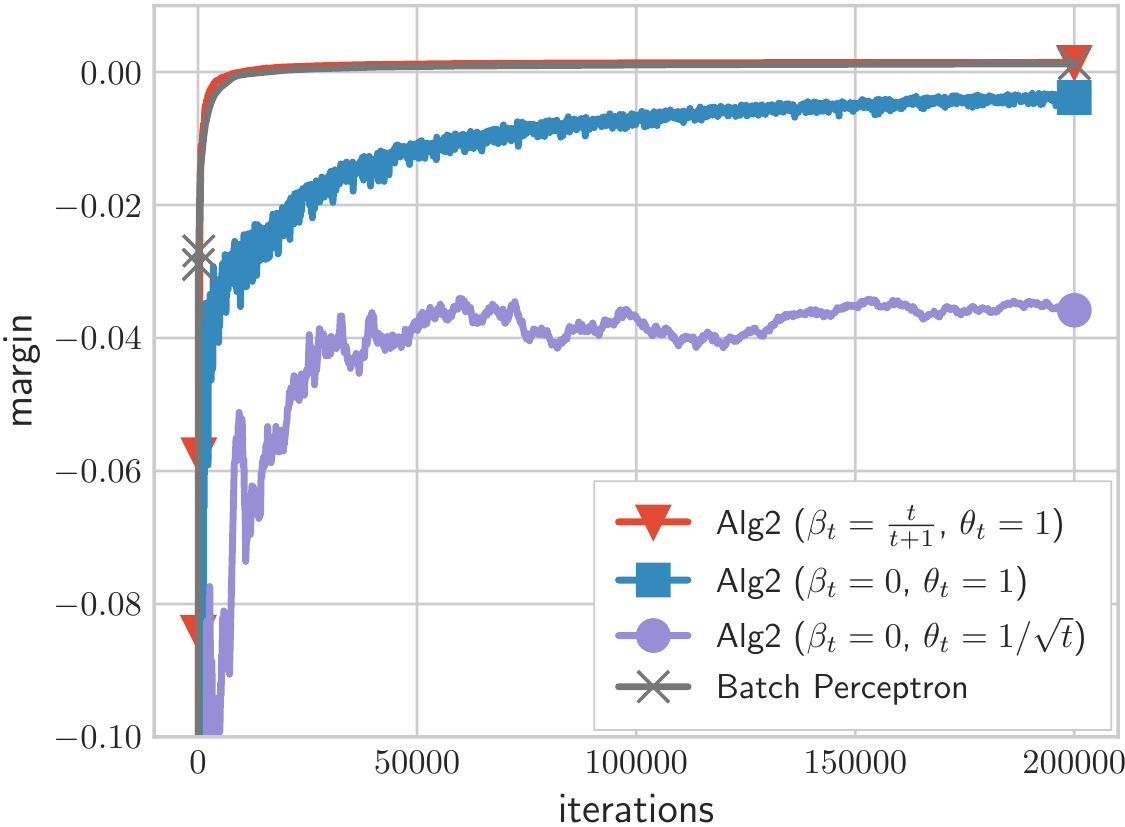}%
  }
  \subfigure[Test error.]{%
    \includegraphics[width=\linewidth]{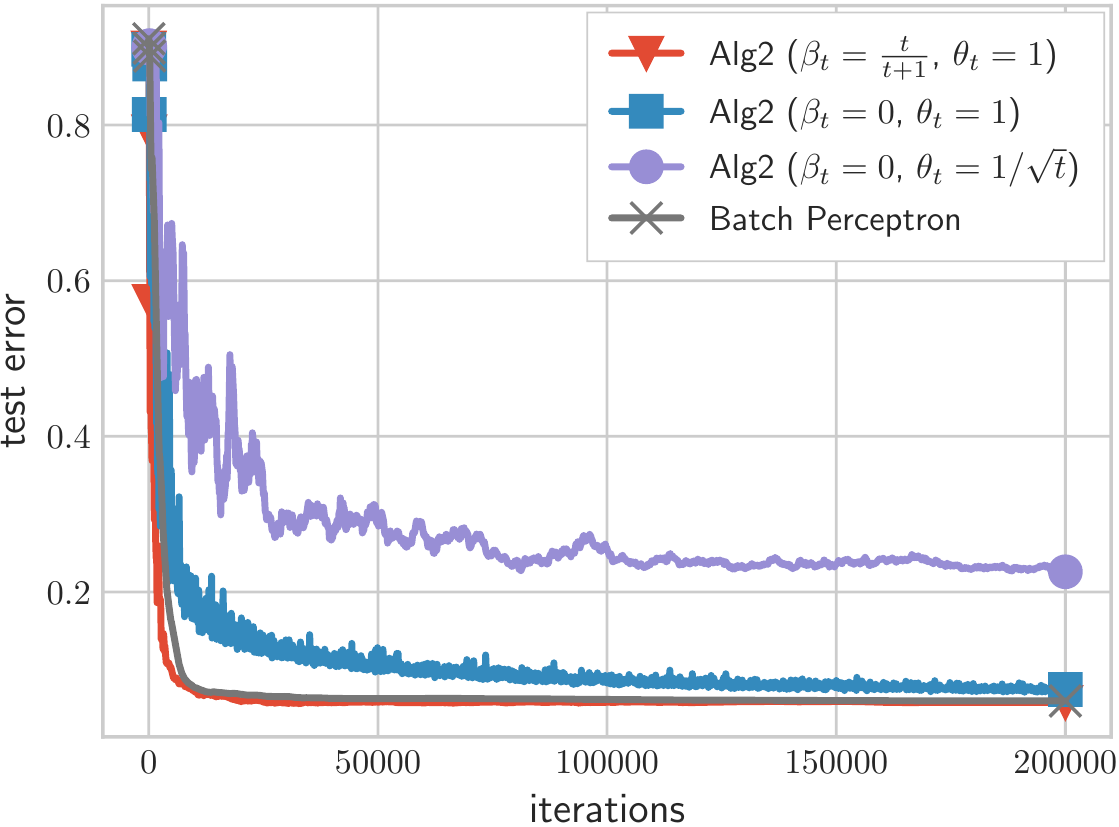}%
  }
  \caption{Margin maximization performance of various methods requiring $\cO(n)$ kernel
    evaluations per iteration.  The batch perceptron is slightly beaten by
    \Cref{alg:adaptron} using the momentum and step size parameters from
    \Cref{alg:momentron}, which is only provided here as a heuristic.  By contrast, the
    theoretically-justified parameters, as analyzed in \Cref{fact:primal_margin_sgd},
    are slower than batch perceptron.
    The data here is the full \texttt{mnist} data, with features given by the initial kernel
    of a 2-homogeneous network of width 128 (cf.~\Cref{sec:app:emp}).}
  \label{fig:sep:adapt}
\end{figure}

Although we do not have a convergence analysis for \Cref{alg:adaptron} with a
nonzero momentum, it works well in practice, as verified on the full
\texttt{mnist} data, shown in \Cref{fig:sep:adapt}.
Still with $\beta_t=t/(t+1)$, \Cref{alg:adaptron} can slightly beat the batch
perceptron method, which is the fastest prior algorithm in the hard-margin
kernel SVM setting.
(Other classical methods, such as stochastic dual coordinate ascent
\citep{sdca}, are focused on the nonseparable soft-margin SVM setting.)

\section{Other Exponentially-Tailed Losses}\label{sec:discussion}

Here we discuss the extension to other exponentially-tailed losses,
such as the logistic loss in the case of binary classification,
and to multiclass losses.

\subsection{Binary Classification}\label{sec:exp_tail}

In previous sections, we focused on the exponential loss.
Our methods can also be applied to other strictly decreasing losses,
such as the logistic loss $\ell(z):=\ln(1+e^z)$, simply by replacing
$Z^\top q_t$ in \Cref{alg:momentron} with $Z^\top\nabla\psi(Zw_t)$, where
$\psi$ is still defined by \cref{eq:psi}.

In the proof of \Cref{fact:momentron}, we only use two properties of
$\psi$: (i) $\psi$ is $\rho$-smooth with respect to the $\ell_\infty$ norm, and
(ii) $\enVert{\nabla\psi}_1\ge1$.
These two properties hold with $\rho=1$ for the exponential loss, and with
$\rho=n$ for the logistic loss \citep[Lemma 5.3, Lemma D.1]{refined_pd}.
Therefore we can use the same analysis to prove a $\tcO(1/t^2)$ margin
maximization rate for the logistic loss; details are given in
\Cref{app_sec:momentron}.

However, the margin rate would additionally depend on $\rho$, which is $n$ for
the logistic loss.
Such a bad dependency on $n$ is probably due to the aggressive initial step size:
from \cref{eq:psi_grad}, we know that $\nabla\psi(Zw_t)$ is just $\nR(w_t)$
normalized by $\ell'\del{\psi(Zw)}/n$.
However, this quantity is at most $1/n$ for the logistic loss, even at
initialization.
It is an interesting open problem to find a better initial step size.

\subsection{Multiclass Classification}\label{sec:multiclass}

Suppose now that inputs $(x_i)_{i=1}^N$ have multiclass labels $(c_i)_{i=1}^N$,
meaning $c_i \in \{1,\ldots, k\}$.  The standard approach to multiclass linear prediction
associates a linear predictor $u_j$ for each class $j\in \{1,\ldots,k\}$; collecting
these as columns of a matrix $U\in \R^{d\times k}$, the multiclass prediction is
\[
  x \mapsto \argmax_{c\in\{1,\ldots,k\}} x^\top U\ve_c,
\]
and letting $\|U\|_\tF$ denote the Frobenius norm,
the margin of $U$ and maximum margin are respectively
\begin{align*}
  \mulgamma(U)
  &:=
  \frac{
  \min_i
  \min_{c \neq c_i}
\del{  x^\top U\ve_{c_i} -  x^\top U \ve_c }}
  {\|U\|_\tF},
  \\
  \bmgamma
  &:=
  \max_{\|U\|_\tF \leq 1}
  \mulgamma(U),
\end{align*}
with edge case $\mulgamma(0) = 0$ as before.

We now show how to reduce this case to the binary case and allow the application
of \Cref{alg:momentron} and its analysis in \Cref{fact:momentron}.
The standard construction
of multiclass losses uses exactly the differences of labels as in the preceding definition
of $\mulgamma$ \citep{zhang_multiclass,tewari_bartlett_multiclass};
that is, define a multiclass risk as
\[
  \mulcR(U) = \frac 1 N \sum_{i=1}^N \sum_{j \neq c_i}
  \ell\del{ x_i^\top U \ve_j - x_i^\top U\ve_{c_i} }.
\]
To rewrite this in our notation as a prediction problem defined by a single matrix $Z$,
define $n := N(k-1)$,
let $F : \R^{d\times k} \to \R^{dk}$ be any fixed flattening of a $d\times k$ matrix into a
vector of length $dk$,
and let $\pi : \{ 1,\ldots, N \} \times \{ 1, \ldots, k-1 \} \to \{1,\ldots,n\}$ be any bijection
between the $N$ original examples and their $n$ new fake counterparts defined as follows:
for each
example $i$ and incorrect label $j\neq c_i$, define $z_{\pi(i,j)} := x_i (\ve_{c_i} - \ve_{j})^\top/\sqrt{2}$,
and let $Z\in\R^{n \times dk}$ be the matrix where row $\pi(i,j)$ is the flattening
$F(z_{\pi(i,j)})^\top$; then, equivalently,
\[
  \frac 1 {k-1} \mulcR(U) = \frac 1 n
\sum_{i=1}^n \ell\del{ F(U)^\top F(z_{\pi(i,j)}) }.
\]
In particular, it suffices to consider a flattened weight vector $w = F(U) \in \R^{dk}$,
and invoke the algorithm and analysis from \Cref{sec:momentron},
with the preceding matrix $Z$.

\begin{theorem}
  \label{fact:multiclass}
  Let a multiclass problem $\{(x_i,c_i)\}_{i=1}^N$ be given with
  maximum multiclass margin $\bmgamma > 0$. Then the corresponding matrix $Z$ as defined
  above has binary margin $\bar\gamma := \bmgamma / \sqrt{2} > 0$.
  Moreover, letting $w_t$ denote the output of \Cref{alg:momentron} when run on this $Z$
  as in \Cref{fact:momentron}, meaning exponential loss $\ell$ and $\beta_t := t / (t+1)$
  and $\theta_t := 1$,
  for every $t\geq 1$
  the un-flattened output $U_t := F^{-1}(w_t)$
  satisfies
  \[
    \mulgamma(U_t)
    \geq
    \bmgamma-\frac{4\del{1+\ln(n)}\del{1+2\ln(t+1)}}{\bmgamma(t+1)^2}.
  \]
\end{theorem}

Due to proceeding by reduction, the guarantees of \Cref{sec:adapt} also hold for an analogous
multiclass version of \Cref{alg:adaptron}.  Indeed, \Cref{alg:adaptron}, with the aggressive
(heuristic) parameters $\beta_t = t/(t+1)$ and $\theta_t = 1$ proved effective in practice,
and was used in the experiments of \Cref{fig:sep:adapt}, as well as the upcoming
\Cref{fig:kernel_evolution}.

One issue that arises in these reduction-based implementations is avoiding explicitly writing
down $Z$ or even individual rows of $Z$, which have $dk$ elements.  Instead, note that sampling
from $q$ as in \Cref{alg:adaptron} now returns both an example index $i$, as well as an incorrect
label $j\neq c_i$.  From here, updates to just the two columns of $U$ corresponding to
$j$ and $c_i$ can be constructed.

\section{Application: Deep Network Kernel Evolution}\label{sec:emp}

As an application of these fast margin-maximization methods, we study the evolution
of kernels encountered during deep network training.  Specifically, consider
the \texttt{cifar10} dataset, which has 50,000 input images in 10 classes; a standard
deep network architecture for this problem is the AlexNet \citep{imagenet_sutskever},
which has both convolutional, dense linear, and various nonlinear layers.

Let $v_t$ denote the AlexNet parameters encountered at epoch $t$
of training on \texttt{cifar10} with a standard stochastic gradient method,
and let $A(x;v_t)$ denote the prediction of AlexNet on input $x$
with these parameters $v_t$.
From here, we can obtain a feature vector $\nabla_{v} A(x;v_t)$, and use it to
construct a matrix $Z$ to plug in to our methods; when $t=0$, this corresponds to the
Neural Tangent Kernel (NTK)
\citep{jacot_ntk,li_liang_nips,du_2_opt}, but here we are also interested in later kernels,
meaning $t>0$, each of which are sometimes called an \emph{after kernel}
\citep{long_afterkernel}, and which in the homogeneous case are known
to converge to a single limiting kernel \citep{dir_align}.
(To handle multiclass output, we simply flatten the Jacobian; as another technical point,
we $\ell_2$-normalize the features to further simplify training and the selection of step sizes.)

For any fixed $t$, we thus obtain a linear prediction problem with rows of matrix $Z$
given by the features $\nabla_{v} A(x;v_t)$ (with additional care for class labels, as in the reductions defined in \Cref{sec:multiclass}),
and can use \Cref{alg:adaptron} to quickly
determine the maximum margin.  \Cref{fig:kernel_evolution:b} presents an experiment
that is consistent with standard beliefs: as $t$ increases, the test error of the corresponding
maximum-margin (kernel) predictor decreases.  In these experiments, the AlexNet training is run
until the features converge, and the test error of the final maximum-margin kernel predictor
is identical to that of the final deep network.

A more interesting example is given in \Cref{fig:kernel_evolution:c}: a case where
feature learning does not help.  All that differs
between \Cref{fig:kernel_evolution:b} and \Cref{fig:kernel_evolution:c} is the choice of random
seed.

A key point is that the AlexNet in both experiments was trained with only 128 training points (the
testing set had the usual 10,000 images, but test error is unsurprisingly large).
The idea is that the feature learning implicit in
deep network training can overfit with such small amounts of data.

Of course, 128 examples is not a standard deep learning regime; these figures merely illustrate
that feature learning may fail, not that it always fails.  It is an interesting open
question to study this phenomenon in realistic scenarios.

\begin{figure}[t!]
  \centering
  \vspace{0.2em}
  \subfigure[Random seed $100$.]{%
  \label{fig:kernel_evolution:b}%
    \includegraphics[width=\linewidth]{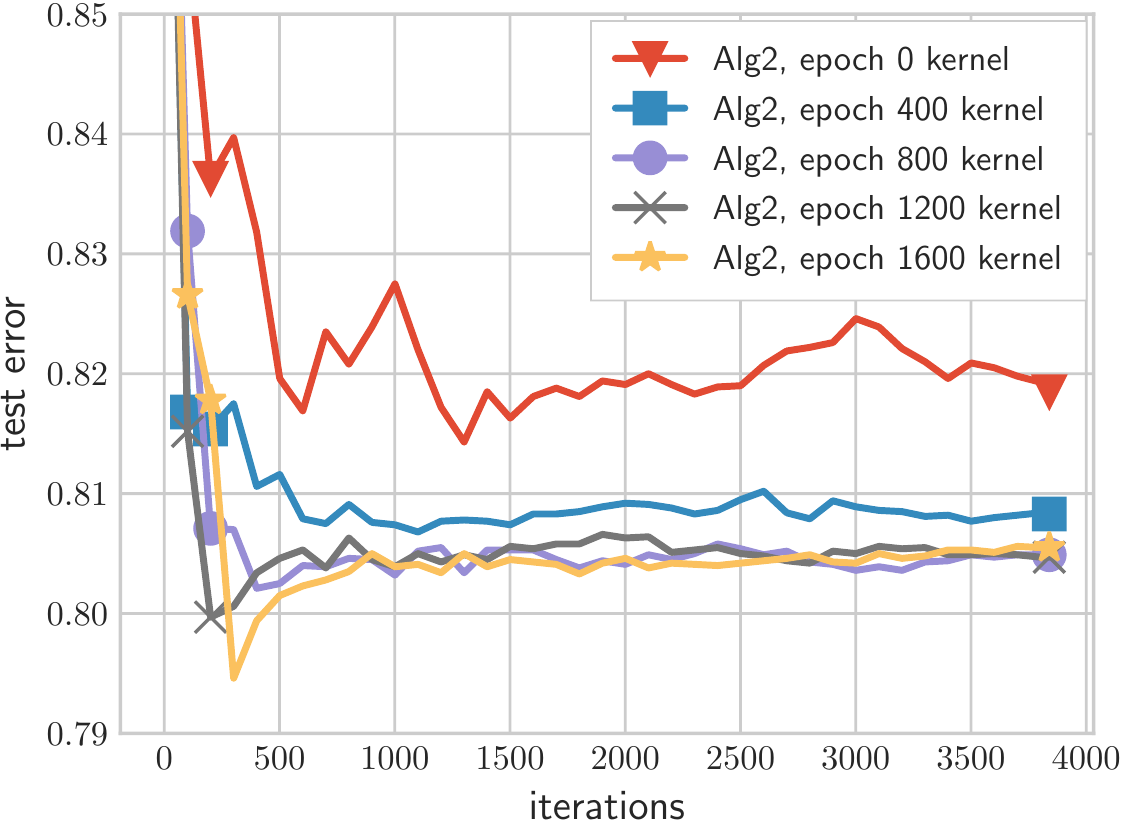}%
  }
  \subfigure[Random seed $13579$.]{%
  \label{fig:kernel_evolution:c}%
    \includegraphics[width=\linewidth]{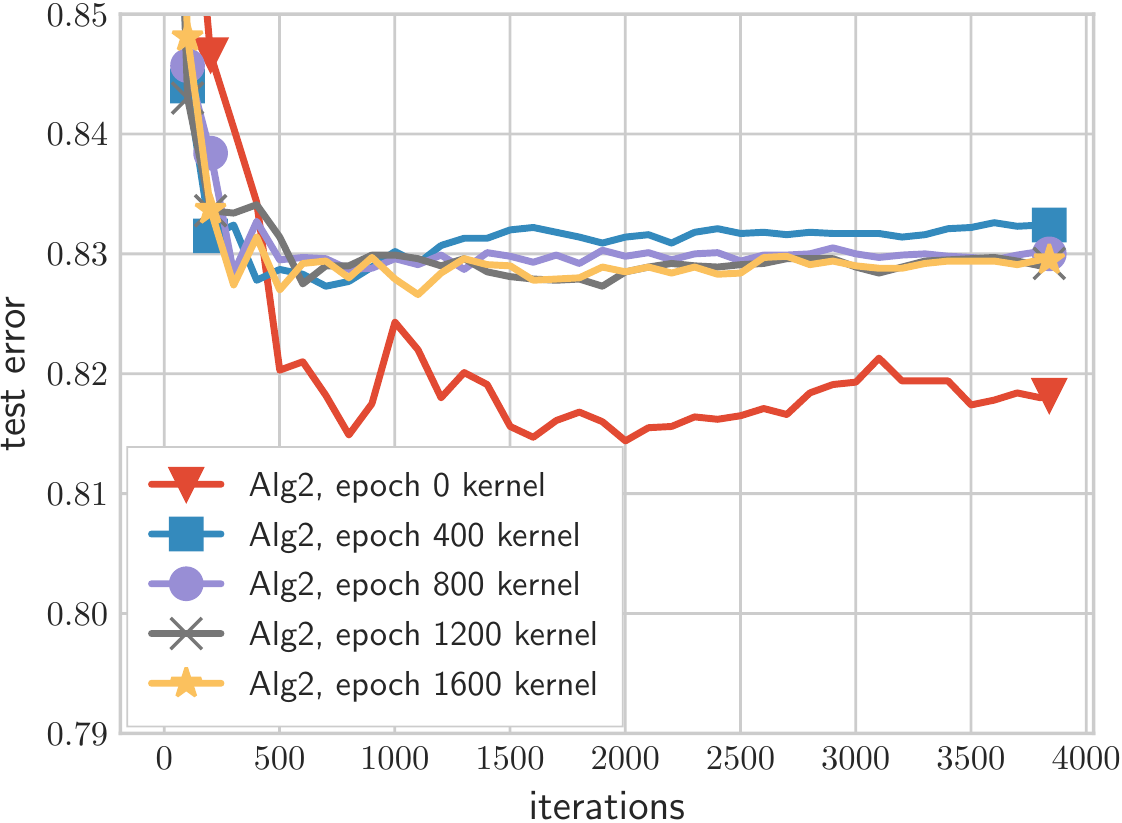}%
  }
  \vspace{-1em}
  \caption{Test error curves of kernel predictors trained with \Cref{alg:adaptron},
    using kernels from different epochs of standard deep network training.  Please see
    \Cref{sec:emp,sec:app:emp}
    for details; the short summary is that changing the random seed
  suffices to change whether kernel features improve or not.}
  \label{fig:kernel_evolution}
\end{figure}

\section{Concluding Remarks and Open Problems}
\label{sec:open}

In this work, we gave two new algorithms based on a dual perspective of margin maximization
and implicit bias: a momentum-based method in \Cref{sec:momentron} constructed via translating
dual Nesterov acceleration iterates into the primal, and an adaptive sampling method in
\Cref{sec:adapt} which aims for greater per-iteration efficiency in the kernel case.

Turning first to \Cref{alg:momentron},
its derivation exposes a connection between Nesterov acceleration in the dual and
momentum in the primal.
Does this connection exist more generally, namely in other optimization problems?

A second open problem is to formally analyze \Cref{alg:adaptron} with
momentum.
As demonstrated empirically in \Cref{fig:sep:adapt}, it can work well,
whereas our analysis disables momentum.

On the empirical side, the small-scale experiments of \Cref{sec:emp}
scratched the surface of situations where feature learning can fail.
Can this phenomenon be exhibited in more realistic scenarios?

\section*{Acknowledgements}

We thank the reviewers for their comments.
ZJ and MT are grateful for support from the NSF under grant IIS-1750051,
and from NVIDIA under a GPU grant.

\bibliography{bib}
\bibliographystyle{icml2021}

\appendix
\onecolumn

\section{Margins in the Primal and in the Dual}
\label{sec:dual}

For completeness, we explicitly derive the convex dual to the primal margin maximization problem,
which is also explored in prior work.

\begin{lemma}\label{fact:margin_pd}
  It holds that
  \begin{align*}
    \bargamma:=\max_{\|w\|_2\le1}\min_{1\le i\le n}(Zw)_i=\min_{q\in\Delta_n}\enVert{Z^\top q}_2.
  \end{align*}
  In the separable case, $\bargamma>0$, and there exists a unique primal optimal
  solution $\baru$, such that for all dual optimal solution $\barq$, it holds
  that $-Z^\top\barq=\bargamma\baru$.
\end{lemma}
\begin{proof}
  Given a convex set $C$, let $\iota_C$ denote the indicator function, i.e.,
  $\iota_C(x)=0$ if $x\in C$, and $\iota_C(x)=\infty$ if $x\not\in C$.
  We note the following convex conjugate pairs:
  \begin{align*}
    \iota_{\Delta_n}^*(v)
    &= \sup_{u\in\Delta_n} \ip{v}{u} = \max_{1\le i\le n} v_i,\\
    (\|\cdot\|_2)^*(q)
    &= \iota_{\|\cdot\|_2 \leq 1}(q).
  \end{align*}
  This gives the Fenchel strong duality \citep[Theorem 3.3.5]{borwein_lewis}
  \begin{align*}
    \min \del{\|Z^\top q\|_2 + \iota_{\Delta_n}(q)}
    &= \max - \iota_{\|\cdot\|_2 \leq 1}(-w) - \iota_\Delta^*(Zw)
    \\
    &= \max\cbr{ - \max_i (Zw)_i : \|w\|_2 \leq 1 }
    \\
    &= \max\cbr{ \min_i (Zw)_i : \|w\|_2 \leq 1 }.
  \end{align*}
  Moreover, for any optimal primal-dual pair $(\baru,\barq)$, we have
  $Z^\top\barq\in\partial\del{\iota_{\|\cdot\|_2 \leq 1}}(-\baru)$, meaning
  $-Z^\top\barq$ and $\baru$ have the same direction.
  Since $\enVert{Z^\top\barq}_2=\bargamma$, we have
  $-Z^\top\barq=\bargamma\baru$.
  The uniqueness of $\baru$ is ensured by \citep[Lemma A.1]{min_norm}.
\end{proof}

\section{A Unified Analysis of Normal and Accelerated Mirror Descent / Dual
Averaging}\label{app_sec:unified}

Consider a convex function $f$, and a convex set $C$, such that $f$ is defined
and $1$-smooth with respect to norm $\|\cdot\|$ on $C$.
Moreover, suppose $\omega:C\to\R$ is differentiable, closed, proper, and
$\alpha$-strongly convex with respect to the same norm $\|\cdot\|$.
We maintain three sequences $q_t,\mu_t,\nu_t$: initialize $\mu_0=q_0\in C$, and
for $t\ge0$, let
\begin{equation}\label{eq:agd_update}
  \begin{split}
    \nu_t & :=(1-\lambda_t)\mu_t+\lambda_tq_t, \\
    q_{t+1} & :=\argmin_{q\in C}\del{f(q_t)+\ip{\nf(\nu_t)}{q-q_t}+\frac{\lambda_t}{\alpha\theta_t}D_\omega(q,q_t)}, \\
    \mu_{t+1} & :=(1-\lambda_t)\mu_t+\lambda_tq_{t+1},
  \end{split}
\end{equation}
where $\lambda_t,\theta_t\in(0,1]$, and
$D_\omega(q,q'):=\omega(q)-\omega(q')-\ip{\nabla\omega(q')}{q-q'}$ denotes the
Bregman distance.

The above update to $q_t$ resembles the mirror descent update.
We can instead use a dual-averaging update, which does not require
differentiability of $\omega$: first note that since $\omega$ is strongly
convex, its convex conjugate $\omega^*$ is smooth
\citep[lemma 2.19]{shalev_online}, and thus defined and differentiable on the
whole Euclidean space.
For any initialization $p_0$, let $q_0:=\nabla\omega^*(p_0)$ and for $t\ge0$,
let
\begin{align}\label{eq:da_agd_update}
  p_{t+1}:=p_t-\frac{\alpha\theta_t}{\lambda_t}\nf(\nu_t),\quad\textup{and}\quad q_{t+1}:=\nabla\omega^*(p_{t+1}).
\end{align}
Note that it is exactly the original update to $q_t$ if for all $t\ge0$ and
$q\in C$, we define
$D_\omega(q,q_t):=\omega(q)-\omega(q_t)-\langle p_t,q-q_t\rangle$.
Below we will analyze this dual-averaging-style update.

The following result is crucial to our analysis.
When $\theta_t=1$, it is basically \citep[eq. (24)]{tseng_agd}, and choosing a
proper $\lambda_t$ would give us acceleration; on the other hand, we further
handle the case of $\lambda_t=1$, when it becomes the usual convergence result
for dual averaging.
\begin{lemma}\label{fact:agd_f_full}
  If $\lambda_t,\theta_t\in(0,1]$ for all $t\ge0$, then for all $t\ge1$ and
  $q\in C$,
  \begin{align*}
    \frac{\theta_{t-1}}{\lambda_{t-1}^2}\del{f(\mu_t)-f(q)}+\sum_{j=1}^{t-1}\del{\frac{\theta_{j-1}}{\lambda_{j-1}^2}-\frac{\theta_j(1-\lambda_j)}{\lambda_j^2}}\del{f(\mu_j)-f(q)}\le & \ \frac{1}{\alpha}\del{D_\omega(q,q_0)-D_\omega(q,q_t)} \\
     & \ +\frac{\theta_0(1-\lambda_0)}{\lambda_0^2}\del{f(\mu_0)-f(q)}.
  \end{align*}
\end{lemma}

To prove \Cref{fact:agd_f_full}, we first recall the following standard result
on mirror descent.
\begin{lemma}\label{fact:md_3_point}
  For all $t\ge0$ and $q\in C$,
  \begin{align*}
    \ip{p_t-p_{t+1}}{q_{t+1}-q}=D_\omega(q,q_t)-D_\omega(q,q_{t+1})-D_\omega(q_{t+1},q_t).
  \end{align*}
\end{lemma}
\begin{proof}
  Note that
  \begin{align*}
    D_\omega(q,q_t) & =\omega(q)-\omega(q_t)-\langle p_t,q-q_t\rangle, \\
    D_\omega(q,q_{t+1}) & =\omega(q)-\omega(q_{t+1})-\langle p_{t+1},q-q_{t+1}\rangle, \\
    D_\omega(q_{t+1},q_t) & =\omega(q_{t+1})-\omega(q_{t})-\langle p_t,q_{t+1}-q_t\rangle.
  \end{align*}
  The proof is finished by direct calculation.
\end{proof}

Now we are ready to prove \Cref{fact:agd_f_full}.
\begin{proof}[Proof of \Cref{fact:agd_f_full}]
  For any $t\ge0$ and $q\in C$,
  \begin{align}
    f(\nu_t)-f(q) & \le\ip{\nf(\nu_t)}{\nu_t-q} \nonumber \\
     & =\ip{\nf(\nu_t)}{\nu_t-q_t}+\ip{\nf(\nu_t)}{q_t-q} \nonumber \\
     & =\frac{1-\lambda_t}{\lambda_t}\ip{\nf(\nu_t)}{\mu_t-\nu_t}+\ip{\nf(\nu_t)}{q_t-q} \nonumber \\
     & \le \frac{1-\lambda_t}{\lambda_t}\del{f(\mu_t)-f(\nu_t)}+\ip{\nf(\nu_t)}{q_t-q}. \label{eq:agd_tmp1}
  \end{align}
  Moreover,
  \begin{align}
    \ip{\nf(\nu_t)}{q_t-q} & =\ip{\nf(\nu_t)}{q_t-q_{t+1}}+\ip{\nf(\nu_t)}{q_{t+1}-q} \nonumber \\
     & =\ip{\nf(\nu_t)}{q_t-q_{t+1}}+\frac{\lambda_t}{\alpha\theta_t}\langle p_t-p_{t+1},q_{t+1}-q_t\rangle \nonumber \\
     & =\ip{\nf(\nu_t)}{q_t-q_{t+1}}-\frac{\lambda_t}{\alpha\theta_t}D_\omega(q_{t+1},q_t)+\frac{\lambda_t}{\alpha\theta_t}\del{D_\omega(q,q_t)-D_\omega(q,q_{t+1})}, \label{eq:agd_md}
  \end{align}
  where we use \Cref{fact:md_3_point} in the last step.
  Next by $1$-smoothness of $f$ and $\alpha$-strong convexity of $\omega$, we
  have
  \begin{align}
    f(\mu_{t+1})-f(\nu_t) & \le\ip{\nf(\nu_t)}{\mu_{t+1}-\nu_t}+\frac{1}{2}\|\mu_{t+1}-\nu_t\|^2 \nonumber \\
     & =\lambda_t\ip{\nf(\nu_t)}{q_{t+1}-q_t}+\frac{\lambda_t^2}{2}\|q_{t+1}-q_t\|^2 \nonumber \\
     & \le\lambda_t\ip{\nf(\nu_t)}{q_{t+1}-q_t}+\frac{\lambda_t^2}{2\theta_t}\|q_{t+1}-q_t\|^2 \nonumber \\
     & \le\lambda_t\ip{\nf(\nu_t)}{q_{t+1}-q_t}+\frac{\lambda_t^2}{\alpha\theta_t}D_\omega(q_{t+1},q_t), \nonumber
  \end{align}
  and therefore
  \begin{align}\label{eq:agd_tmp2}
    \ip{\nf(\nu_t)}{q_t-q_{t+1}}-\frac{\lambda_t}{\alpha\theta_t}D_\omega(q_{t+1},q_t)\le \frac{1}{\lambda_t}\del{f(\nu_t)-f(\mu_{t+1})}.
  \end{align}
  Then \cref{eq:agd_tmp1,eq:agd_md,eq:agd_tmp2} imply
  \begin{align}
    f(\nu_t)-f(q) & \le\ip{\nf(\nu_t)}{\nu_t-q} \nonumber \\
     & \le \frac{1-\lambda_t}{\lambda_t}\del{f(\mu_t)-f(\nu_t)}+\frac{1}{\lambda_t}\del{f(\nu_t)-f(\mu_{t+1})}+\frac{\lambda_t}{\alpha\theta_t}\del{D_\omega(q,q_t)-D_\omega(q,q_{t+1})}, \label{eq:agd_f_iter}
  \end{align}
  and rearranging terms gives
  \begin{align*}
    \frac{1}{\lambda_t}\del{f(\mu_{t+1})-f(q)}-\frac{1-\lambda_t}{\lambda_t}\del{f(\mu_t)-f(q)}\le \frac{\lambda_t}{\alpha\theta_t}\del{D_\omega(q,q_t)-D_\omega(q,q_{t+1})}.
  \end{align*}
  Multiply both sides by $\theta_t/\lambda_t$, we have
  \begin{align}\label{eq:agd_tmp3}
    \frac{\theta_t}{\lambda_t^2}\del{f(\mu_{t+1})-f(q)}-\frac{\theta_t(1-\lambda_t)}{\lambda_t^2}\del{f(\mu_t)-f(q)}\le \frac{1}{\alpha}\del{D_\omega(q,q_t)-D_\omega(q,q_{t+1})}.
  \end{align}
  Taking the sum of \cref{eq:agd_tmp3} from step $0$ to $t-1$ finishes the
  proof.
\end{proof}

Next we invoke \Cref{fact:agd_f_full} to get concrete rates.
We further make the following constraint on $\lambda_t$: let
\begin{align}\label{eq:lambda_cond}
  \lambda_0:=1,\quad\textup{and}\quad\frac{1}{\lambda_t^2}-\frac{1}{\lambda_t}\le \frac{1}{\lambda_{t-1}^2}\textup{ for all }t\ge1.
\end{align}
Note that by this construction,
\begin{align}\label{eq:inv_lambda_sum}
  \frac{1}{\lambda_t^2}\le \frac{1}{\lambda_0^2}+\sum_{j=1}^{t}\frac{1}{\lambda_t}=\sum_{j=0}^{t}\frac{1}{\lambda_t}.
\end{align}
\begin{theorem}\label{fact:agd_f}
  With \cref{eq:lambda_cond} satisfied and $\theta_t=1$, for all $t\ge1$ and $\barq\in\argmin_{q\in C}f(q)$,
  \begin{align*}
    f(\mu_t)-f(\barq)\le\frac{\lambda_{t-1}^2}{\alpha}\del{D_\omega(\barq,q_0)-D_\omega(\barq,q_t)}\le\frac{\lambda_{t-1}^2}{\alpha}D_\omega(\barq,q_0).
  \end{align*}
  In particular, if $\lambda_t=2/(t+2)$, then
  \begin{align*}
    f(\mu_t)-f(\barq)\le \frac{4D_\omega(\barq,q_0)}{\alpha(t+1)^2}.
  \end{align*}
\end{theorem}
\begin{proof}
  For $\barq\in\argmin_{q\in C}f(q)$, we have
  $f(\mu_j)-f(\barq)\ge0$.
  It then follows from \Cref{fact:agd_f_full} and \cref{eq:lambda_cond} and
  $\lambda_0=1$ that
  \begin{align*}
    \frac{1}{\lambda_{t-1}^2}\del{f(\mu_t)-f(\barq)}\le \frac{1}{\alpha}\del{D_\omega(\barq,q_0)-D_\omega(\barq,q_t)}.
  \end{align*}
\end{proof}

\section{Omitted Proofs from \Cref{sec:momentron}}\label{app_sec:momentron}

Here we prove the results in \Cref{sec:momentron}.
We consider a slightly more general setting: recall $\psi(\xi)$ is defined as
\begin{align*}
  \psi(\xi):=\ell^{-1}\del{\sum_{i=1}^{n}\ell(\xi_i)},
\end{align*}
where $\ell$ is a strictly increasing loss $\ell:\R\to\R$ with
$\lim_{z\to-\infty}\ell(z)=0$ and $\lim_{z\to\infty}\ell(z)=\infty$, and thus
$\psi$ is well-defined.
It follows directly that for any $\xi\in\R^n$, we have $\nabla\psi(\xi)>0$,
since
\begin{align}\label{eq:psi_grad_pos}
  \nabla\psi(\xi)_i=\frac{\ell'(\xi_i)}{\ell'\del{\ell^{-1}\del{\sum_{j=1}^{n}\ell(\xi_j)}}}>0.
\end{align}

We assume $\psi$ is $\rho$-smooth with respect to the $\ell_\infty$ norm; this
is true for the exponential loss with $\rho=1$, and true for the logistic loss
with $\rho=n$ \citep[Lemma 5.3]{refined_pd}.

On the dual, we run \cref{eq:agd_update} with $\|\cdot\|=\|\cdot\|_1$, and
$f(q)=\phi(q)=\enVert{Z^\top q}_2^2/2$, and $\omega=\psi^*$, and
$\alpha=1/\rho$.
It holds that $\phi$ is $1$-smooth with respect to the $\ell_1$ norm
\citep[Lemma 2.5]{refined_pd}, and since $\psi$ is $\rho$-smooth with respect to
the $\ell_\infty$ norm, we have $\psi^*$ is $(1/\rho)$-strongly convex with
respect to the $\ell_1$ norm \citep[lemma 2.19]{shalev_online}.
On the other hand, the primal iterate is updated as follows: let $w_0:=0$, and
for $t\ge0$, let
\begin{align*}
  w_{t+1}:=w_t-\frac{\theta_t}{\rho\lambda_t}Z^\top\nu_t.
\end{align*}
Note that if we let $p_t=Zw_t$, then
\begin{align*}
  p_{t+1}=p_t-\frac{\theta_t}{\rho\lambda_t}ZZ^\top\nu_t,
\end{align*}
therefore by \cref{eq:da_agd_update}, $\nabla\psi(Zw_t)=\nabla\psi(p_t)=q_t$.

We first prove the following general version of \Cref{fact:dual_phi_simple}.
\begin{lemma}
  With $\theta_t=1$ and $\lambda_t=2/(t+2)$, for all $t\ge1$ and
  $\barq\in\argmin_{q\in\Delta_n}\phi(q)$,
  \begin{align*}
    \phi(\mu_t)-\phi(\barq)\le \frac{4\rho D_{\psi^*}(\barq,q_0)}{(t+1)^2}.
  \end{align*}
\end{lemma}
\begin{proof}
  We just need to apply \Cref{fact:agd_f}.
  Specifically, for the exponential loss, $\rho=1$, and $D_{\psi^*}$ is the KL
  divergence, and moreover $D_{\psi^*}(\barq,q_0)\le\ln(n)$ since $q_0$ is the
  uniform distribution.
\end{proof}

Next we prove a general version of \Cref{fact:wt_characterization_simple}.
\begin{lemma}
  For all $\lambda_t,\theta_t\in(0,1]$, if $\lambda_0=1$, then for all $t\ge0$,
  \begin{align}\label{eq:wt_momentum}
    w_{t+1}=w_t-\frac{\theta_t}{\rho}\del{g_t+Z^\top q_t},
  \end{align}
  where $g_0:=0$ and for all $t\ge1$,
  \begin{align}\label{eq:momentum_def}
    g_t:=\frac{\lambda_{t-1}(1-\lambda_t)}{\lambda_t}\del{g_{t-1}+Z^\top q_t}.
  \end{align}
  In addition, it holds for all $t\ge1$ that
  \begin{align}\label{eq:mu_momentum}
    Z^\top \mu_t=\lambda_{t-1}\del{g_{t-1}+Z^\top q_t}.
  \end{align}

  Specifically, for $\lambda_t=2/(t+2)$, it holds that
  \begin{align*}
    \frac{\lambda_{t-1}(1-\lambda_t)}{\lambda_t}=\frac{t}{t+1},\quad\textup{and}\quad g_t=\sum_{j=1}^{t}\frac{j}{t+1}Z^\top q_j,\quad\textup{and}\quad Z^\top\mu_t=\frac{2g_t}{t}.
  \end{align*}
\end{lemma}
\begin{proof}
  To prove \cref{eq:wt_momentum}, we only need to show that $g_t$ defined by
  \cref{eq:momentum_def} satisfies
  \begin{align*}
    g_t=\frac{w_t-w_{t+1}}{\theta_t/\rho}-Z^\top q_t=Z^\top\del{\frac{1}{\lambda_t}\nu_t-q_t}.
  \end{align*}
  It holds at $t=0$ by definition, since $\lambda_0=1$ and $\nu_0=q_0$.
  Moreover,
  \begin{align*}
    \frac{1}{\lambda_{t+1}}\nu_{t+1}-q_{t+1}=\frac{1-\lambda_{t+1}}{\lambda_{t+1}}\mu_{t+1} & =\frac{1-\lambda_{t+1}}{\lambda_{t+1}}\del{\nu_t+\lambda_t(q_{t+1}-q_t)} \\
     & =\frac{\lambda_t(1-\lambda_{t+1})}{\lambda_{t+1}}\del{\frac{1}{\lambda_t}\nu_t-q_t}+\frac{\lambda_t(1-\lambda_{t+1})}{\lambda_{t+1}}q_{t+1},
  \end{align*}
  which coincides with the recursive definition of $g_t$.

  For \cref{eq:mu_momentum}, it is true by definition when $t=1$, since
  $\lambda_0=1$, and $g_0=0$, and by definition $\mu_1=q_1$.
  For $t\ge1$, by \cref{eq:agd_update} and the inductive hypothesis,
  \begin{align*}
    Z^\top\mu_{t+1} & =(1-\lambda_t)Z^\top\mu_t+\lambda_tZ^\top q_{t+1} \\
     & =(1-\lambda_t)\lambda_{t-1}(g_{t-1}+Z^\top q_t)+\lambda_tZ^\top q_{t+1} \\
     & =\lambda_t \frac{(1-\lambda_t)\lambda_{t-1}}{\lambda_t}(g_{t-1}+Z^\top q_t)+\lambda_tZ^\top q_{t+1} \\
     & = \lambda_t\del{g_t+Z^\top q_{t+1}}.
  \end{align*}

  For $\lambda_t=2/(t+2)$, it can be verified directly that
  $\lambda_t(1-\lambda_{t+1})/\lambda_{t+1}=(t+1)/(t+2)$.
  The explicit expression of $g_t$ clearly holds when $t=0$; for $t\ge0$,
  \begin{align*}
    g_{t+1}:=\frac{t+1}{t+2}\del{g_t+Z^\top q_{t+1}}=\frac{t+1}{t+2}\sum_{j=1}^{t}\frac{j}{t+1}Z^\top q_j+\frac{t+1}{t+2}Z^\top q_{t+1}=\sum_{j=1}^{t+1}\frac{j}{t+2}Z^\top q_j.
  \end{align*}
  For $Z^\top\mu_t$, we just need to invoke
  \cref{eq:momentum_def,eq:mu_momentum} and note that
  $\lambda_t/(1-\lambda_t)=2/t$.
\end{proof}

Next we prove a general version of \Cref{fact:-psi_lb_agd_simple}.
\begin{lemma}\label{fact:-psi_lb_agd}
  Let $\theta_t=1$ for all $t\ge0$, and $\lambda_0=1$, then for all $t\ge1$,
  \begin{align*}
    -\psi(Zw_t)\ge & \ -\psi(Zw_0)+\frac{1}{2\rho\lambda_{t-1}^2}\enVert{Z^\top\mu_t}_2^2 \\
     & \ +\sum_{j=1}^{t-1}\frac{1}{2\rho}\del{\frac{1}{\lambda_{j-1}^2}-\frac{1-\lambda_j}{\lambda_j^2}}\enVert{Z^\top\mu_j}_2^2 \\
     & \ +\sum_{j=0}^{t-1}\frac{1}{2\rho\lambda_j}\enVert{Z^\top\nu_j}_2^2.
  \end{align*}
\end{lemma}
\begin{proof}
  Note that by \cref{eq:agd_f_iter},
  \begin{align*}
    \ip{\nphi(\nu_t)}{\nu_t-\barq} & \le \frac{1-\lambda_t}{\lambda_t}\del{\phi(\mu_t)-\phi(\nu_t)}+\frac{1}{\lambda_t}\del{\phi(\nu_t)-\phi(\mu_{t+1})}+\rho\lambda_t\del{D_{\psi^*}(\barq,q_t)-D_{\psi^*}(\barq,q_{t+1})} \\
     & =\phi(\nu_t)+\frac{1-\lambda_t}{\lambda_t}\phi(\mu_t)-\frac{1}{\lambda_t}\phi(\mu_{t+1})+\rho\lambda_t\del{D_{\psi^*}(\barq,q_t)-D_{\psi^*}(\barq,q_{t+1})}.
  \end{align*}
  Moreover, $\ip{\nphi(\nu_t)}{\nu_t}=\enVert{Z^\top\nu_t}_2^2=2\phi(\nu_t)$,
  and thus
  \begin{align}\label{eq:-psi_lb_tmp1}
    \phi(\nu_t)-\ip{\nphi(\nu_t)}{\barq}\le \frac{1-\lambda_t}{\lambda_t}\phi(\mu_t)-\frac{1}{\lambda_t}\phi(\mu_{t+1})+\rho\lambda_t\del{D_{\psi^*}(\barq,q_t)-D_{\psi^*}(\barq,q_{t+1})}.
  \end{align}
  Additionally, let $p_t=Zw_t$, we have
  \begin{align}
    D_{\psi^*}(\barq,q_t)-D_{\psi^*}(\barq,q_{t+1}) & =\psi^*(\barq)-\psi^*(q_t)-\langle p_t,\barq-q_t\rangle-\psi^*(\barq)+\psi^*(q_{t+1})+\langle p_{t+1},\barq-q_{t+1}\rangle \nonumber \\
     & =\langle p_t,q_t\rangle-\psi^*(q_t)-\langle p_{t+1},q_{t+1}\rangle+\psi^*(q_{t+1})-\langle p_t-p_{t+1},\barq\rangle \nonumber \\
     & =\psi(p_t)-\psi(p_{t+1})-\langle p_t-p_{t+1},\barq\rangle \nonumber \\
     & =\psi(Zw_t)-\psi(Zw_{t+1})-\frac{1}{\rho\lambda_t}\langle\nphi(\nu_t),\barq\rangle \label{eq:-psi_lb_tmp2}
  \end{align}
  Therefore \cref{eq:-psi_lb_tmp1,eq:-psi_lb_tmp2} imply
  \begin{align}\label{eq:-psi_lb_tmp3}
    \psi(Zw_t)-\psi(Zw_{t+1})\ge \frac{1}{\rho\lambda_t^2}\phi(\mu_{t+1})-\frac{1-\lambda_t}{\rho\lambda_t^2}\phi(\mu_t)+\frac{1}{\rho\lambda_t}\phi(\nu_t).
  \end{align}
  Take the sum of \cref{eq:-psi_lb_tmp3} from $0$ to $t-1$ finishes the proof.
\end{proof}

Next we prove a general version of \Cref{fact:wt_norm_simple}.
\begin{lemma}\label{fact:wt_norm}
  Let $\theta_t=1$ for all $t\ge0$, and suppose
  $\enVert{\nabla\psi(Zw_t)}_1\ge1$, then
  \begin{align*}
    \sum_{j=0}^{t-1}\frac{\bargamma}{\rho\lambda_j}\le\|w_t\|_2\le \sum_{j=0}^{t-1}\frac{1}{\rho\lambda_j}\enVert{Z^\top\nu_j}_2.
  \end{align*}
\end{lemma}
\begin{proof}
  The upper bound follows immediately from the triangle inequality.
  For the lower bound, recall $\baru$ denotes the maximum-margin classifier,
  \begin{align*}
    \|w_t\|_2\ge \langle w_t,\baru\rangle & =\sum_{j=0}^{t-1}\frac{1}{\rho\lambda_j}\ip{-Z^\top\nu_j}{\baru} \\
     & =\sum_{j=0}^{t-1}\frac{1}{\rho\lambda_j}\ip{\nu_j}{-Z\baru} \\
     & \ge\sum_{j=0}^{t-1}\frac{\bargamma}{\rho\lambda_j},
  \end{align*}
  since $\nu_j>0$, and $\enVert{\nu_j}_1\ge1$, and
  $\langle-z_i,\baru\rangle\ge\bargamma$ for all $i$.
\end{proof}

To prove \Cref{fact:momentron}, we need the following result which gives an
alternative characterization of $w_t$ using $\mu_j$.
\begin{lemma}\label[lemma]{fact:agd_wt_alt}
  Let $\theta_t=1$, for all $t\ge1$, we have
  \begin{align*}
    w_t=\frac{1}{\rho}Z^\top q_t-\frac{1}{\rho}Z^\top q_0-\sum_{j=0}^{t-1}\frac{1}{\rho\lambda_j}Z^\top\mu_{j+1},
  \end{align*}
  and if $\lambda_t=2/(t+2)$, then
  \begin{align*}
    \frac{1}{2\rho\lambda_{t-1}^2}\enVert{Z^\top\mu_t}_2^2+\sum_{j=1}^{t-1}\frac{1}{2\rho}\del{\frac{1}{\lambda_{j-1}^2}-\frac{1-\lambda_j}{\lambda_j^2}}\enVert{Z^\top\mu_j}_2^2\ge \sum_{j=0}^{t-1}\frac{1}{2\rho\lambda_j}\enVert{Z^\top\mu_{j+1}}_2^2-2\ln(n)\ln(t+1).
  \end{align*}
\end{lemma}
\begin{proof}
  Note that by construction,
  $\frac{1}{\lambda_t}\nu_t=\frac{1-\lambda_t}{\lambda_t}\mu_t+q_t$, and thus
  \begin{align*}
    w_t=-\frac{1}{\rho}\sum_{j=0}^{t-1}Z^\top\del{\frac{1}{\lambda_j}\nu_j} & =-\frac{1}{\rho}\sum_{j=0}^{t-1}Z^\top\del{\frac{1-\lambda_j}{\lambda_j}\mu_j+q_j} \\
     & =-\frac{1}{\rho}Z^\top q_0+\frac{1}{\rho}Z^\top q_t-\frac{1}{\rho}\sum_{j=0}^{t-1}Z^\top\del{\frac{1-\lambda_j}{\lambda_j}\mu_j+q_{j+1}} \\
     & =-\frac{1}{\rho}Z^\top q_0+\frac{1}{\rho}Z^\top q_t-\frac{1}{\rho}\sum_{j=0}^{t-1}Z^\top\del{\frac{1}{\lambda_j}\mu_{j+1}}. \\
  \end{align*}

  On the second claim, note that
  \begin{align*}
    \frac{1}{2\rho\lambda_{t-1}^2}\enVert{Z^\top\mu_t}_2^2+\sum_{j=1}^{t-1}\frac{1}{2\rho}\del{\frac{1}{\lambda_{j-1}^2}-\frac{1-\lambda_j}{\lambda_j^2}}\enVert{Z^\top\mu_j}_2^2 & \ge \frac{1}{2\rho\lambda_{t-1}^2}\bargamma^2+\sum_{j=1}^{t-1}\frac{1}{2\rho}\del{\frac{1}{\lambda_{j-1}^2}-\frac{1-\lambda_j}{\lambda_j^2}}\bargamma^2 \\
     & =\sum_{j=0}^{t-1}\frac{1}{2\rho\lambda_j}\bargamma^2.
  \end{align*}
  Additionally, \Cref{fact:agd_f} implies
  \begin{align*}
    \frac{1}{2\rho\lambda_j}\del{\enVert{Z^\top\mu_{j+1}}_2^2-\bargamma^2}\le\lambda_j D_{\psi^*}(\barq,q_0)\le\lambda_j\ln(n).
  \end{align*}
  Therefore
  \begin{align*}
     & \ \frac{1}{2\rho\lambda_{t-1}^2}\enVert{Z^\top\mu_t}_2^2+\sum_{j=1}^{t-1}\frac{1}{2\rho}\del{\frac{1}{\lambda_{j-1}^2}-\frac{1-\lambda_j}{\lambda_j^2}}\enVert{Z^\top\mu_j}_2^2 \\
    \ge & \ \sum_{j=0}^{t-1}\frac{1}{2\rho\lambda_j}\bargamma^2 \\
    = & \ \sum_{j=0}^{t-1}\frac{1}{2\rho\lambda_j}\enVert{Z^\top\mu_{j+1}}_2^2-\sum_{j=0}^{t-1}\frac{1}{2\rho\lambda_j}\del{\enVert{Z^\top\mu_{j+1}}_2^2-\bargamma^2} \\
    \ge & \ \sum_{j=0}^{t-1}\frac{1}{2\rho\lambda_j}\enVert{Z^\top\mu_{j+1}}_2^2-\ln(n)\sum_{j=0}^{t-1}\lambda_j,
  \end{align*}
  and note that
  \begin{align*}
    \sum_{j=0}^{t-1}\lambda_j=\sum_{j=0}^{t-1}\frac{2}{j+2}\le2\ln(t+1).
  \end{align*}
\end{proof}

Now we can prove \Cref{fact:momentron}.
Note that here we need \Cref{fact:wt_norm}, and particularly
$\enVert{\nabla\psi(Zw_t)}_1\ge1$; this is true for the exponential loss since
$\nabla\psi\in\Delta_n$, and it is also true for the logistic loss
\citep[Lemma D.1]{refined_pd}.
\begin{proof}[Proof of \Cref{fact:momentron}]
  \Cref{fact:-psi_lb_agd,fact:agd_wt_alt} imply
  \begin{align*}
    \psi(Zw_0)-\psi(Zw_t)\ge \sum_{j=0}^{t-1}\frac{1}{2\rho\lambda_j}\enVert{Z^\top\mu_{j+1}}_2^2+\sum_{j=0}^{t-1}\frac{1}{2\rho\lambda_j}\enVert{Z^\top\nu_j}_2^2-2\ln(n)\ln(t+1).
  \end{align*}
  Therefore
  \begin{align}
    \margin{w_t} & \ge\frac{\psi(Zw_0)-\psi(Zw_t)}{\|w_t\|_2}-\frac{\psi(Zw_0)}{\|w_t\|_2} \nonumber \\
     & \ge \frac{\sum_{j=0}^{t-1}\frac{1}{2\rho\lambda_j}\enVert{Z^\top\mu_{j+1}}_2^2}{\|w_t\|_2}+\frac{\sum_{j=0}^{t-1}\frac{1}{2\rho\lambda_j}\enVert{Z^\top\nu_j}_2^2}{\|w_t\|_2}-\frac{2\ln(n)\ln(t+1)}{\|w_t\|_2}-\frac{\ln(n)}{\|w_t\|_2} \nonumber \\
     & =\frac{\sum_{j=0}^{t-1}\frac{1}{2\rho\lambda_j}\enVert{Z^\top\mu_{j+1}}_2^2}{\|w_t\|_2}+\frac{\sum_{j=0}^{t-1}\frac{1}{2\rho\lambda_j}\enVert{Z^\top\nu_j}_2^2}{\|w_t\|_2}-\frac{\ln(n)\del{1+2\ln(t+1)}}{\|w_t\|_2}. \label{eq:momentron_tmp1}
  \end{align}
  By the triangle inequality and the alternative characterization of $w_t$ in
  \Cref{fact:agd_wt_alt}, we have
  \begin{align*}
    \|w_t\|_2\le \frac{1}{\rho}\enVert{Z^\top q_0}_2+\frac{1}{\rho}\enVert{Z^\top q_t}_2+\sum_{j=0}^{t-1}\frac{1}{\rho\lambda_j}\enVert{Z^\top\mu_{j+1}}_2\le \frac{2}{\rho}+\sum_{j=0}^{t-1}\frac{1}{\rho\lambda_j}\enVert{Z^\top\mu_{j+1}}_2.
  \end{align*}
  Therefore
  \begin{align*}
    \frac{\sum_{j=0}^{t-1}\frac{1}{2\rho\lambda_j}\enVert{Z^\top\mu_{j+1}}_2^2}{\|w_t\|_2} & \ge \frac{\bargamma \sum_{j=0}^{t-1}\frac{1}{2\rho\lambda_j}\enVert{Z^\top\mu_{j+1}}_2}{\frac{2}{\rho}+\sum_{j=0}^{t-1}\frac{1}{\rho\lambda_j}\enVert{Z^\top\mu_{j+1}}_2} \\
     & =\frac{\bargamma}{2}\del{1-\frac{2}{2+\sum_{j=0}^{t-1}\frac{1}{\lambda_j}\enVert{Z^\top\mu_{j+1}}_2}} \\
     & \ge \frac{\bargamma}{2}\del{1-\frac{2}{\sum_{j=0}^{t-1}\frac{1}{\lambda_j}\enVert{Z^\top\mu_{j+1}}_2}} \\
     & \ge \frac{\bargamma}{2}\del{1-\frac{8}{\bargamma(t+1)^2}}=\frac{\bargamma}{2}-\frac{4}{(t+1)^2}.
  \end{align*}
  where we use
  $\sum_{j=0}^{t-1}\frac{1}{\lambda_j}\enVert{Z^\top\mu_{j+1}}_2\ge \sum_{j=0}^{t-1}\frac{\bargamma}{\lambda_j}\ge \frac{\bargamma(t+1)^2}{4}$.
  The remaining part of \cref{eq:momentron_tmp1} can be handled in the same way
  as in the proof of \Cref{fact:momentron_half}.

  The second part of \Cref{fact:momentron} is proved at the end of
  \Cref{sec:pd_update}.
\end{proof}

\section{Omitted Proofs from \Cref{sec:adapt}}\label{app_sec:adapt}

Here we prove \Cref{fact:primal_margin_sgd}.
We need the following two results.
The first one gives a lower bound on $-\psi(Zw_t)$, which is an approximation of
the true unnormalized margin.
\begin{lemma}\label{fact:primal_margin_sgd_num}
  Under the conditions of \Cref{fact:primal_margin_sgd}, it holds with
  probability $1-\nicefrac{\delta}{2}$ that
  \begin{align*}
    -\psi(Zw_t)\ge \sum_{j=0}^{t-1}\theta_j\enVert{Z^\top q_j}_2^2-2\ln(n)-\sqrt{8\ln\del{\frac{2}{\delta}}\ln(n)}.
  \end{align*}
\end{lemma}
\begin{proof}
  Since $\psi$, the ln-sum-exp function, is $1$-smooth with respect to the
  $\ell_\infty$ norm, we have
  \begin{align*}
    \psi(Zw_{t+1})-\psi(Zw_t) & \le \langle q_t,Zw_{t+1}-Zw_t\rangle+\enVert{Z(w_{t+1}-w_t)}_\infty^2 \\
     & =-\theta_t\ip{Z^\top q_t}{z_{i_t}}+\theta_t^2\enVert{Zz_{i_t}}_\infty^2 \\
     & \le-\theta_t\ip{Z^\top q_t}{z_{i_t}}+\theta_t^2\enVert{z_{i_t}}_2^2\le-\theta_t\ip{Z^\top q_t}{z_{i_t}}+\theta_t^2. \\
  \end{align*}
  Therefore
  \begin{align*}
    -\psi(Zw_t)\ge-\psi(Zw_0)+\sum_{j=0}^{t-1}\theta_j\ip{Z^\top q_j}{z_{i_j}}-\sum_{j=0}^{t-1}\theta_j^2.
  \end{align*}
  Additionally,
  \begin{align*}
    \envert{\ip{Z^\top q_j}{Z^\top q_j-z_{i_j}}}\le2,
  \end{align*}
  and Azuma's inequality implies, with probability $1-\nicefrac{\delta}{2}$,
  \begin{align*}
    \sum_{j=0}^{t-1}\theta_j\ip{Z^\top q_j}{Z^\top q_j-z_{i_j}}\le\sqrt{8\ln\del{\frac{2}{\delta}}\sum_{j=0}^{t-1}\theta_j^2}.
  \end{align*}
  Consequently, letting $\theta_j=\sqrt{\ln(n)/t}$, with probability
  $1-\nicefrac{\delta}{2}$,
  \begin{align*}
    -\psi(Zw_t) & \ge-\psi(Zw_0)+\sum_{j=0}^{t-1}\theta_j\enVert{Z^\top q_j}_2^2-\sqrt{8\ln\del{\frac{2}{\delta}}\sum_{j=0}^{t-1}\theta_j^2}-\sum_{j=0}^{t-1}\theta_j^2 \\
     & =\sum_{j=0}^{t-1}\theta_j\enVert{Z^\top q_j}_2^2-2\ln(n)-\sqrt{8\ln\del{\frac{2}{\delta}}\ln(n)}.
  \end{align*}
\end{proof}

On the other hand, we give upper and lower bounds on $\|w_t\|_2$, the
normalization term.
\begin{lemma}\label{fact:primal_margin_sgd_den}
  Under the conditions of \Cref{fact:primal_margin_sgd}, it holds with
  probability $1-\nicefrac{\delta}{2}$ that
  \begin{align*}
    \|w_t\|_2\le \sum_{j=0}^{t-1}\theta_j\enVert{Z^\top q_j}+\sqrt{\frac{8\ln(n)}{\delta}},
  \end{align*}
  and it always holds that $\|w_t\|_2\ge\bargamma\sqrt{t\ln(n)}$.
\end{lemma}
\begin{proof}
  Define
  \begin{align*}
    r_t:=\sum_{j=0}^{t-1}\theta_j\del{z_{i_j}-Z^\top q_j}.
  \end{align*}
  Note that
  \begin{align*}
    \bbE\sbr{\|r_t\|_2^2\middle|q_0,\ldots,q_{t-2}} & =\bbE\sbr{\|r_{t-1}\|_2^2+\theta_j\ip{r_{t-1}}{z_{i_{t-1}}-Z^\top q_{t-1}}+\theta_j^2\enVert{z_{i_{t-1}}-Z^\top q_{t-1}}_2^2\middle|q_0,\ldots,q_{t-2}} \\
     & =\|r_{t-1}\|_2^2+\bbE\sbr{\theta_j^2\enVert{z_{i_{t-1}}-Z^\top q_{t-1}}_2^2\middle|q_0,\ldots,q_{t-2}} \\
     & \le\|r_{t-1}\|_2^2+4\theta_j^2,
  \end{align*}
  thus
  \begin{align*}
    \bbE\sbr{\|r_t\|_2^2}\le\bbE\sbr{\|r_{t-1}\|_2^2}+4\theta_j^2,
  \end{align*}
  and $\bbE\sbr{\|r_t\|_2^2}\le4\theta_j^2t=4\ln(n)$.
  By Markov's inequality, with probability $1-\nicefrac{\delta}{2}$, it holds
  that $\|r_t\|_2^2\le8\ln(n)/\delta$.
  In this case, we have
  \begin{align*}
    \|w_t\|_2=\enVert{\sum_{j=0}^{t-1}\theta_jz_{i_j}}_2\le\enVert{\sum_{j=0}^{t-1}\theta_jZ^\top q_j}_2+\sqrt{\frac{8\ln(n)}{\delta}}\le \sum_{j=0}^{t-1}\theta_j\enVert{Z^\top q_j}+\sqrt{\frac{8\ln(n)}{\delta}}.
  \end{align*}
  For the lower bound, just note that
  \begin{align*}
    \|w_t\|_2\ge \langle w_t,\baru\rangle=\sum_{j=0}^{t-1}\theta_j \langle-z_{i_j},\baru\rangle\ge\sum_{j=0}^{t-1}\theta_j\bargamma=\bargamma\sqrt{t\ln(n)}.
  \end{align*}
\end{proof}

Now we are ready to prove \Cref{fact:primal_margin_sgd}.
\begin{proof}[Proof of \Cref{fact:primal_margin_sgd}]
  First note the inequality $\margin{w_t}\ge-\psi(Zw_t)/\|w_t\|_2$.
  Then \Cref{fact:primal_margin_sgd_num,fact:primal_margin_sgd_den} imply, with
  probability $1-\delta$,
  \begin{align*}
    \margin{w_t}\ge \frac{-\psi(Zw_t)}{\|w_t\|_2} & \ge \frac{\sum_{j=0}^{t-1}\theta_j\enVert{Z^\top q_j}_2^2-2\ln(n)-\sqrt{8\ln\del{\frac{2}{\delta}}\ln(n)}}{\|w_t\|_2} \\
     & =\frac{\sum_{j=0}^{t-1}\theta_j\enVert{Z^\top q_j}_2^2+\bargamma\sqrt{\frac{8\ln(n)}{\delta}}}{\|w_t\|_2}-\frac{2\ln(n)+\sqrt{8\ln\del{\frac{2}{\delta}}\ln(n)}+\bargamma\sqrt{\frac{8\ln(n)}{\delta}}}{\|w_t\|_2} \\
     & \ge \frac{\sum_{j=0}^{t-1}\theta_j\enVert{Z^\top q_j}_2^2+\bargamma\sqrt{\frac{8\ln(n)}{\delta}}}{\sum_{j=0}^{t-1}\theta_j\enVert{Z^\top q_j}_2+\sqrt{\frac{8\ln(n)}{\delta}}}-\frac{2\ln(n)+\sqrt{8\ln\del{\frac{2}{\delta}}\ln(n)}+\bargamma\sqrt{\frac{8\ln(n)}{\delta}}}{\bargamma\sqrt{t\ln(n)}}.
  \end{align*}
  Since $\enVert{Z^\top q_j}_2\ge\bargamma$, it follows that
  \begin{align*}
    \margin{w_t}\ge\bargamma-\frac{2\sqrt{\ln(n)}+\sqrt{8\ln\del{\frac{2}{\delta}}}}{\bargamma\sqrt{t}}-\sqrt{\frac{8}{\delta t}}.
  \end{align*}
  Letting
  \begin{align*}
    t=\max\del{\left\lceil \frac{32\ln(n)+64\ln(2/\delta)}{\bargamma^2\epsilon^2}\right\rceil,\left\lceil\frac{32}{\delta\epsilon^2}\right\rceil}
  \end{align*}
  finishes the proof.
\end{proof}

\section{Reducing Multiclass to Binary Classification}

Here we verify the reduction to the binary case.

\begin{proof}[Proof of \Cref{fact:multiclass}]
  We first show a property than the first one given in the statement,
  namely that $\gamma(F(U)) = \mulgamma(U)/\sqrt{2}$ for any $U\in\R^{d\times k}$;
  from this it follows directly that
  \[
    \bargamma
    =
    \max_{\|w\|_2\leq 1} \gamma(w)
    =
    \max_{\|U\|_\tF\leq 1} \gamma(F(U))
    =
    \max_{\|U\|_\tF\leq 1} \frac{ \mulgamma(U)}{\sqrt 2}
    =
    \frac{\bmgamma}{\sqrt 2}
  \]
  To this end, for any $U\in\R^{d\times k}$, the case $U=0$ follows directly since
  $\gamma(F(0)) = 0 = \mulgamma(0)$, and when $U\neq 0$ then
  \begin{align*}
    \gamma(F(U))
    &=
    \min_{i\in\{1,\ldots,n\}} \ve_i^\top Z F(U)
    \\
    &=
    \min_{i\in\{1,\ldots,N\}} \min_{j\neq c_i} \ip{ z_{\pi(i,j)} }{ U }
    \\
    &=
    \frac 1 {\sqrt 2} \min_{i\in\{1,\ldots,N\}} \min_{j\neq c_i}
    \ip{x_i (\ve_{c_i} - \ve_j)^\top }{ U }
    \\
    &=
    \frac 1 {\sqrt 2} \min_{i\in\{1,\ldots,N\}} \min_{j\neq c_i}
    \del{x_i^\top U \ve_{c_i} - x_i^\top U \ve_j}
    \\
    &= \frac {\mulgamma(U)}{\sqrt 2}.
  \end{align*}
  From here, the algorithmic guarantee is direct from \Cref{fact:momentron}, which
  can be applied due to the $\sqrt{2}$ factor in the definition of $z_{\pi(i,c)}$,
  which insures $\|F(z_{\pi(i,c)})\|_2\leq 1$:
  \begin{align*}
    \mulgamma(U_t)
    &=
    \sqrt{2} \gamma(F(U_t))
    \\
    &\geq
    \sqrt{2}\sbr{ \bargamma - \frac {4(1+\ln(n))(1+2\ln(1+t))}{\bargamma(t+1)^2} }
    \\
    &\geq
    \bmgamma - \frac {4(1+\ln(n))(1+2\ln(1+t))}{\bmgamma(t+1)^2}.
  \end{align*}
\end{proof}

Note that if the reduction included all $k$ classes and not $k-1$, then the margin
equivalence would fail, since there would be a term $\ip{x_i (\ve_{c_i} - \ve_{c_i})^\top}{U}
= \ip{0}{U} = 0$, giving a margin of zero for any data.

\section{Further Experimental Details}%
\label{sec:app:emp}

\Cref{fig:sep:batch,fig:nonsep:batch,fig:sep:adapt} use
the standard \texttt{mnist} data, which has 60,000 training data and
10,000 testing data across $10$ classes, with inputs in $\R^{784}$.

\begin{itemize}
  \item
    \Cref{fig:sep:batch} restricts the data to digits $0$ and $1$, which
    leads to a separable problem.  All methods are run with standard parameters,
    meaning in particular that normalized and unnormalized gradient descent have step size $1$,
    \Cref{alg:momentron} has parameters matching \Cref{fact:momentron}, and batch perceptron is
    implemented exactly as in \citep{batch_perceptron}.

  \item
    \Cref{fig:nonsep:batch} restricts the data to digits $3$ and $5$, which can be verified
    as nonseparable after inspecting the dual objective as discussed in \Cref{sec:momentron}.

  \item
    \Cref{fig:sep:adapt} uses multiclass versions of \Cref{alg:adaptron} as detailed
    in \Cref{sec:multiclass}, for instance avoiding writing down any explicit vectors of
    dimension $dk$.  The data is the full \texttt{mnist}, meaning all $10$ classes, which
    are made into a linearly separable problem by considering the NTK features given
    by a 2-homogeneous network, specifically a shallow network with a single hidden layer
    of $128$ ReLU nodes.  This width is in fact more than adequate to produce reliable results;
    experiments with narrower widths were similar.

\end{itemize}

\Cref{fig:kernel_evolution} used
the standard \texttt{cifar10} data, which has 50,000 training data and
10,000 testing data across $10$ classes, with inputs in $\R^{3072}$, but represented
as $32\times 32$ images with $3$ color channels, a format convenient for convolutional layers.
A standard AlexNet was trained on this data \citep{imagenet_sutskever}, where the training
procedure was an unembellished mini-batch stochastic gradient descent.

\end{document}